\documentclass[twoside]{article}

\usepackage[accepted]{aistats2025}
\usepackage[round]{natbib}

\usepackage{graphicx}
\usepackage{booktabs}
\usepackage{amsfonts}
\usepackage{amsmath}
\usepackage{amssymb}
\usepackage{amsthm}
\usepackage{dsfont}
\usepackage{siunitx}
\usepackage{xcolor}
\usepackage{multirow}
\usepackage{booktabs}
\usepackage{bm}
\usepackage{algorithm}
\usepackage{algorithmic}
\usepackage[font=small]{caption}
\usepackage{subcaption}
\usepackage{hyperref}
\usepackage[capitalise]{cleveref}

\newtheorem{theorem}{Theorem}

\newtheorem{definition}{Definition}

\newcommand{\K}{K} %subspace dim
\newcommand{\kk}{k} %index subsp
\newcommand{\D}{D} %fullspace dim
\newcommand{\X}{\mathcal{X}}
\newcommand{\Y}{\mathcal{Y}}
\newcommand{\T}{\mathcal{T}}
\newcommand{\w}{\bm{\theta}}
\newcommand{\varw}{\bm{\vartheta}}
\newcommand{\W}{\bm{\Theta}}
\newcommand{\subspace}{\bm{\Phi}}
\newcommand{\x}{\mathbf{x}}
\newcommand{\p}{p}
\newcommand{\m}{m}
\newcommand{\loss}{\ell}
\newcommand{\lossspace}{\mathcal{L}}
\newcommand{\pathh}{\rho}
\newcommand{\bez}{b}

\newcommand{\tunnel}{\varrho}
\newcommand{\btun}{\bm{\xi}}
\newcommand{\btunsp}{\bm{\Xi}}
\newcommand{\paramset}{\bm{\Lambda}}
\newcommand{\pathlen}{\mathcal{S}}
\newcommand{\tangent}{\bm{\tau}} %OD nutzt T schon für temp
\newcommand{\tortho}{\bm{\kappa}}
\newcommand{\orthosys}{\bm{\Omega}}
\newcommand{\proj}{\bm{\Pi}}

\newcommand{\sselem}{\bm{\varphi}}

\newcommand{\n}{n} %Number of update steps of the Bezier Curve. Epochs if Batchsize = 1
\newcommand{\lK}{\lambda_\K} %Approximation of path length
\newcommand{\rg}{R_g} %Radius of Gyration
\newcommand{\pcom}{\bar{\w}} %Center of Mass
\newcommand{\noise}{\sigma}
\newcommand{\eps}{\bm{\epsilon}}

\newcommand{\var}[1]{\operatorname{Var}\left( #1 \right)}

\newcommand{\re}{R_e} % End-to-End Distance
\newcommand{\norm}[1]{\left\| #1 \right\|} % Norm

\begin{document}

% If your paper is accepted and the title of your paper is very long,
% the style will print as headings an error message. Use the following
% command to supply a shorter title of your paper so that it can be
% used as headings.
%
%\runningtitle{I use this title instead because the last one was very long}

% If your paper is accepted and the number of authors is large, the
% style will print as headings an error message. Use the following
% command to supply a shorter version of the authors names so that
% they can be used as headings (for example, use only the surnames)
%
\runningauthor{Daniel Dold, Julius Kobialka, Nicolai Palm, Emanuel Sommer, David R\"ugamer, Oliver D\"urr}
\twocolumn[

%#2
\aistatstitle{Paths and Ambient Spaces in Neural Loss Landscapes}

%#3 Widening the Perspective on Neural Network Loss Surfaces through Tunnel Vision: A Scalable Approximate Bayesian Inference Method

%#4 Characterising Neural Network Posterior Surfaces through Tunnel Vision: A Scalable Approximate Bayesian Inference Method

%#5 Characterising Neural Network Loss Surfaces through Tunnel Vision: An Approximate Bayesian Inference Method

%#6 Drilling Tunnels into Neural Networks Loss Landscape Guiding Approximate Bayesian Inference

%#6b Drilling Tunnels into Neural Networks Loss Landscape to Guide Approximate Bayesian Inference

%#7 Not just Scratching the (Loss) Surface: On Paths and Tunnels in Neural Networks

%#8 On Path and Tunnel Optimization in Neural Networks

%#9 The Losses of Neural Networks: The Path to Success is Never Straight

%#10 On the Properties of Paths, Tunnels, and Subspaces in Neural (Network) Loss Landscapes

%#11 On the Properties of Paths and Ambient Spaces in Neural Loss Landscapes

%11b Characterizing Paths and Ambient Spaces in Neural Loss Landscapes

%#12 Creating an Energetic Ambience for low-loss Paths 

\aistatsauthor{Daniel Dold\\HTWG Konstanz \And Julius Kobialka\\LMU Munich, MCML \And Nicolai Palm\\LMU Munich, MCML \AND Emanuel Sommer\\LMU Munich, MCML \And David R\"ugamer\\LMU Munich, MCML \And Oliver D\"urr\\HTWG Konstanz, TIDIT.ch}
\aistatsaddress{ } 
]

% \aistatsauthor{Daniel Dold \And Julius Kobialka \And Nicolai Palm }
% \aistatsaddress{ HTWG Konstanz \And  LMU Munich, MCML \And LMU Munich, MCML } 

% \aistatsauthor{Emanuel Sommer \And David Rügamer \And Oliver Dürr}
% \aistatsaddress{LMU Munich, MCML\And LMU Munich, MCML\And HTWG Konstanz, TIDIT.ch } ]

\begin{abstract}
Understanding the structure of neural network loss surfaces, particularly the emergence of low-loss tunnels, is critical for advancing neural network theory and practice. In this paper, we propose a novel approach to directly embed loss tunnels into the loss landscape of neural networks. Exploring the properties of these loss tunnels offers new insights into their length and structure and sheds light on some common misconceptions. We then apply our approach to Bayesian neural networks, where we improve subspace inference by identifying pitfalls and proposing a more natural prior that better guides the sampling procedure. 
\end{abstract}

\section{INTRODUCTION AND RELATED WORK} \label{sec:intro}

Studying the emergence of lower-dimensional connected structures of low-loss in the high-dimensional loss landscape of neural networks is an important research direction fostering a better understanding of neural networks. Much of the literature in this area focuses on the connectedness of optimized networks. This property referred to as \emph{mode connectivity} not only provides valuable insights into the landscape of hypotheses of a neural network, but can also be used to inform the optimization process \citep{ainsworth2023git}, improve sampling approaches in Bayesian neural networks \citep{izmailov2020subspace, dold2024bayesian}, protect against adversarial attacks \citep{Zhao2020Bridging}, improve model averaging \citep{wortsman2022model}, and guide fine-tuning \citep{lubana2023mechanistic}. 

% linear vs nonlinear mode connectivity
Various forms of mode connectivity have therefore been studied. 
The most commonly investigated phenomenon is linear mode connectivity \citep{frankle2020linear,entezari2022the}. Other connectivity hypotheses include linear layerwise connectivity \citep{adilova2024layerwise,wortsman2022model}, quadratic \citep{lubana2023mechanistic} or star-shaped connections \citep{sonthalia2024deep,lin2024exploring}, geodesic mode connectivity \citep{tan2023geodesic}, parametrized curves \citep{garipov2018loss},  manifolds \citep{benton2021loss}, and general minimum energy paths \citep{draxler18a}. The latter, in particular, suggests that there is no ``loss barrier'' between modes if paths are constructed flexible enough. Current findings also suggest that shared invariances in networks induce connectivity of models, while the absence of (linear) connectivity implies dissimilar mechanisms \citep[see, e.g.,][]{lubana2023mechanistic}. 

% symmetries
Apart from characterizing trained networks, some research also tries to improve and better understand mode connectivity by taking neural network properties into account. One of the most common approaches is to account for parameter space symmetries \citep{tatro2020optimizing,entezari2022the,zhao2023understanding}. The initialization and type of architecture are further properties that were found to relate to the way loss valleys emerge \citep{benzing2022random}. 

\textbf{Direct Optimization of Subspaces}\, One promising approach tightly connected to mode connectivity is to embed a certain topological space into the loss landscape of the neural network and directly optimize this space. 
In contrast to the aforementioned literature, this changes the objective of the network training to find a region in the network's parameter space in which relevant model hypotheses reside. 
Using a toy model representing high-dimensional intersecting wedges, \cite{fort2019large} were able to link the two objectives and reproduce some of the loss landscape properties of real neural networks. An alternative approach by \cite{garipov2018loss} and \cite{gotmare2018using} is to directly optimize a path between two fixed modes in the actual network landscape. % \cite{benton2021loss} uses a similar strategy but using a simplex. % but simultaneously enforce the volume of the simplex to be as large as possible. 

\textbf{Future Directions}\, While embedding a manifold of greater complexity into the loss landscape would be a natural extension of previous approaches, this also makes it harder to study and understand hypotheses within it. In contrast, loss tunnels or paths possess favorable properties that can be directly understood, and using a path does not limit the expressiveness of generated hypotheses in function space \citep{draxler18a}. Despite being simpler than a complex manifold embedding, theoretical characteristics and the training of loss tunnels are not completely understood yet.

\subsection{Our Contributions}

This paper broadens the understanding of loss paths and tunnels by making the following contributions:\\
1. We propose a flexible method to directly embed loss tunnels into neural loss landscapes. Our approach can be trained with an arbitrary number of control points while also being modularly applicable to all types of neural networks. In contrast, existing approaches are more invasive in their adoption, e.g., requiring changes in the standard implementation of common layers.\\
2. We provide new insights into the nature of loss paths and tunnels, in particular their length, their optimization, benefits, and other important properties. \\
3. Using these insights and their implementation, we demonstrate how to advance subspace inference in Bayesian neural networks, an application that greatly benefits from such loss tunnels. For this, we also propose a matching and more intuitive prior that allows for better guidance of sampling procedures. % and further makes use of the entire volume surrounding the low-loss path by defining a coordinate system in the tunnel space. 
%Running samplers through these loss tunnels, in turn, provides insights into their nature. 

% \begin{itemize}
%     \item Path vs. Tunnel: Make connection that via subspace sampling we make use of entire volume surrounding the low-loss path. When only talking about optimization of tunnels, the reader would likely expect some kind of higher-dimensional structure that is embedded in the loss surface, as this is the distinction \citet{benton2021loss} makes.
% \end{itemize}

\section{LOSS PATHS AND TUNNELS} \label{sec:pathtunnel}

\subsection{Notation and Objective}

In this work, we consider neural networks $f_{\w}:\X \to \Y$ mapping features $\x \in \X \subseteq \mathbb{R}^\p$ to an outcome space $\Y \subseteq \mathbb{R}^\m$. The network is parametrized with weights $\w \in \W \subseteq \mathbb{R}^\D$, where $\D$ is typically very large and trained to minimize a loss function. In this work, we usually consider the loss function as a continuous function of the parameters: $\loss:\W \to \lossspace \subseteq \mathbb{R}$.

\textbf{Objective}\, Our goal is to construct a lower-dimensional connected structure $\subspace \subseteq \mathbb{R}^\K$ of low loss in the high-dimensional loss surface $\W$, where typically $\K \ll \D$. As discussed in \cref{sec:intro}, we will focus on paths and tunnels in this work. To formalize these structures, we define the following.

\begin{definition}[Loss Path] \label{def:path}
    Let $\T := [0,1]$. We call $\pathh:\T\to\lossspace$ a loss path if there exists a mapping $\bez:\mathcal{T}\to\W$ such that $\pathh:= \loss \circ \bez$ is a continuous function on $\T$.
\end{definition}

In \cref{def:path}, $\bez$ is used as the actual curve embedded into the neural network loss surface. While $\pathh$ is a loss path in the sense that it yields loss values for a given ``time'' $t$, this is just a byproduct of the curve $\bez$ describing an ensemble of model parametrizations on $\T$, whose performance in turn is evaluated with $\loss$.

In contrast to a loss path, defined on $\T$, we consider a loss tunnel to be a multidimensional object.
\begin{definition}[Loss Tunnel] \label{def:tunnel}
    Let $\btunsp \subset \mathbb{R}^{\K-1}$ be a compact space on $\mathbb{R}^{\K-1}$. We call $\tunnel: \mathcal{T}\times\btunsp \to \lossspace$ a $\K$-dimensional loss tunnel if for every $\btun\in\btunsp$, $\tunnel(\cdot,\btun):\T\to\lossspace$ defines a loss path on $\T$.
\end{definition}
In other words, we obtain the loss tunnel by ``lifting'' a path into the product space and for every $t\in\mathcal{T}$ on the path, $\tunnel(t,\cdot)$ describes a manifold (the ambient space) embedded in $\mathbb{R}^{K-1}$. Every combination of values $\btun\in\btunsp$, in turn, defines a loss path in this tunnel. 

Note that while we define loss paths and tunnels on compact spaces, this is merely to convey the idea of low-loss regions in $\W$ similar to the idea of \emph{loss barriers}, and there is typically no actual boundary in this space. %In our application to Bayesian sampling, these sets will, e.g., correspond to ``high-probability regions''.

\subsection{Path Optimization} \label{sec:pathoptim}

%[Mathematical description]

Following \cref{def:tunnel}, the search for a loss tunnel naturally leads to finding a loss path first. A property shared by all existing path approaches is to define the \emph{initial} and \emph{terminal point} of the loss path. A common way to do this is to set $\w_0 := \bez(0)$ and $\w_1 := \bez(1)$, where $\w_0,\w_1\in\W$ represent two optimized neural networks solutions. Given these endpoints, a low-loss path can be constructed by finding a parameterized curve of constantly low-loss between those points \citep{draxler18a, garipov2018loss}. The hope is that this path contains functionally diverse, well-performing models. 

As in previous approaches, we also use a parametrized regular curve $\bez$ to construct the loss path. We, however, opt for a more flexible way to learn this curve, not restricting ourselves to pre-defined endpoints, and thereby also constructing a method that can be adapted more flexibly (cf.~\cref{fig:dold_vs_izmailov}). More specifically, we do not fix the endpoints of $\bez$, but instead, define $\bez$ based on $\K+1$ control points $\w_{0}, \ldots, \w_{\K} \in \W$ that can freely move in $\W$. We then optimize this set of control points $\bm{\Lambda} := \{ \w_{0}, \dots, \w_{\K}\}$ that define the curve $\bez_{\paramset}$ by integrating over the current path loss
\begin{equation}
\label{eq:lossfunc}
    L(\bm{\Lambda}) := \int_0^1 \rho_{\paramset}(t) dt = \int_0^1 \ell(\bez_{\paramset}(t))dt.
\end{equation} 
While optimizing \cref{eq:lossfunc} corresponds to the optimization of a curve with constant ``speed'' $\| \bez_{\paramset}'(t) \|$, an unrestricted approach would be to optimize over its expectation uniformly in the space $\W$ \citep{garipov2018loss}: 
%While optimizing \cref{eq:lossfunc} corresponds to the optimization of a curve with uniformly sampled time, an alternative way would be to compute the expectation uniformly over the curve \citep{garipov2018loss}: 
\begin{equation}
\label{eq:lossfunc_equal}
\tilde{L}(\paramset) := \int_0^1 \!\!\ell(\bez_{\paramset}(t)) \| \bez_{\paramset}'(t) \|dt \cdot \left( \int_0^1 \!\!\| \bez_{\paramset}'(t) \|dt \right)^{-1}\!\!.
\end{equation}
Since we use a parameterized regular curve, which has non-constant speed, we could reparameterize $t$ with the curve length $s$ through the relation $s = \int_0^t \| \bez_{\paramset}'(u) \|du$ and sample $s \sim \text{Uniform}(0, \pathlen)$ instead of $t$ to convert \cref{eq:lossfunc_equal} into \cref{eq:lossfunc}. However, as also discussed in \cite{garipov2018loss}, this approach is computationally intractable due to the dependency of $s$ on $\paramset$. In particular, recovering $t$ from $s$ requires a root solver, which is in general not differentiable. 

\textbf{Importance Sampling}\, While previous approaches \citep{garipov2018loss,izmailov2020subspace,dold2024bayesian} focused on implementing paths using the expectation over $t$ as in \cref{eq:lossfunc}, a better understanding of loss paths also requires investigating the differences to an optimization using \cref{eq:lossfunc_equal}.
To this end, we developed an importance sampling approach that allows optimizing \cref{eq:lossfunc_equal}. Details are given in the following section. As later results will demonstrate, there is negligible difference between the optimization of both objectives \cref{eq:lossfunc,eq:lossfunc_equal}. For this reason, we will thus focus on the objective in \cref{eq:lossfunc} in the following.

\subsection{Practical Implementation}

In practice, we suggest implementing the parametrized curve and, hence, the loss path in a way that allows easy adoption for any network architecture. % and with little computational overhead. 

\textbf{Path parametrization}\, A flexible and recently promoted approach is to use a B\'{e}zier curve \citep[see, e.g.,][]{garipov2018loss}, which we define as %a B\'{e}zier curve
\begin{equation} \label{eq:bezier}
    \bez_{\paramset}(t) = \sum_{\kk=0}^\K \omega_{\kk}(t) \w_{\kk}  := \sum_{\kk=0}^\K \binom{\K}{\kk} (1-t)^{\K-\kk} t^{\kk} \w_{\kk},
\end{equation}
$t\in\T:=[0,1]$. Using $\bez_{\paramset}$, we define our loss path according to \cref{def:path} by setting $\pathh := \loss \circ \bez_{\paramset}$ (since \cref{eq:bezier} maps from $\T$ to $\W$ and $\ell$ maps from $\W$ to $\lossspace$).

\textbf{Loss computation}\, The advantage of a parametrized curve becomes evident when computing the path loss in \cref{eq:lossfunc}. We do this by drawing $M$ samples from $t\sim U(0,1)$ in every forward pass and approximate the expectation $\int_0^1 \ell(\bez_{\paramset}(t))dt = \mathbb{E}_{t\sim U(0,1)}[\ell(\bez_{\paramset}(t))]$ using Monte Carlo. In practice, using $M=1$ is often sufficient. Although increasing $M$ reduces gradient noise, we observed no significant performance improvements with larger values. % (while being computationally faster in every forward pass at the cost of an increased total number of required epochs). %In addition, this allows change backpropagation in a minimally invasive way.

\textbf{Model updates}\, Based on the previous loss computation discussion, we can derive the gradients for model updates for a given time point $t^\ast$ via the chain rule
% \begin{equation}
% \begin{split}
$\nabla_{\paramset}L := \left\{ \frac{\partial{L}}{\partial{\w_0}},\ldots,\frac{\partial{L}}{\partial{\w_{\K}}} \right\}%\\
%&
= 
% \left\{ \frac{\partial{L}}{\partial{\bez_{\paramset}(t^\ast)}} \frac{\partial{\bez_{\paramset}(t^\ast)}}{\partial{\w_0}},\ldots,\frac{\partial{L}}{\partial{\bez_{\paramset}(t^\ast)}} \frac{\partial{\bez_{\paramset}(t^\ast)}}{\partial{\w_{\K}}}\right\},
\left\{ \frac{\partial{L}}{\partial{\bez_{\paramset}(t^\ast)}} \frac{\partial{\bez_{\paramset}(t^\ast)}}{\partial{\w_{k}}}\right\}_{k \color{black}=0,\ldots,K}$,
% \end{split}
% \end{equation}
where the first term $\frac{\partial{L}}{\partial{\bez_{\paramset}(t^\ast)}}$ for each control point is the standard gradient $\nabla_{\varw} \ell$ of the loss function $\ell$ for the given point $\varw := \bez_{\paramset}(t^\ast)$ and the second term $\partial{\bez_{\paramset}(t^\ast)}/{\partial{\w_{\kk}}} \equiv \omega_k(t^\ast)$% simplifies to
%$\binom{\K}{\kk} (1-t^\ast)^{\K-\kk} ({t^\ast})^{\kk}$
. Since the first term is independent of $k$, cached for all $\K+1$ control points, and the second term is a scalar weighting, the model update can be performed inexpensively and efficiently using common auto-differentiation libraries. In particular, this routine \textbf{does not require to build custom layer or network modules}. Additionally, it can be straightforwardly combined with other inference approaches (such as last-layer Bayesian approximations), by only modifying a subset of model parameters using the weighting $\omega_k$. Decoupling the model architecture from parameter computations $\varw$ also aligns well with modern functional deep learning frameworks such as JAX \citep{jax2018github}. We describe this implementation in \cref{alg:subspace_cons}, exemplary using stochastic gradient descent (SGD).
\begin{figure}[ht]
    \centering
    \includegraphics[width=\linewidth]{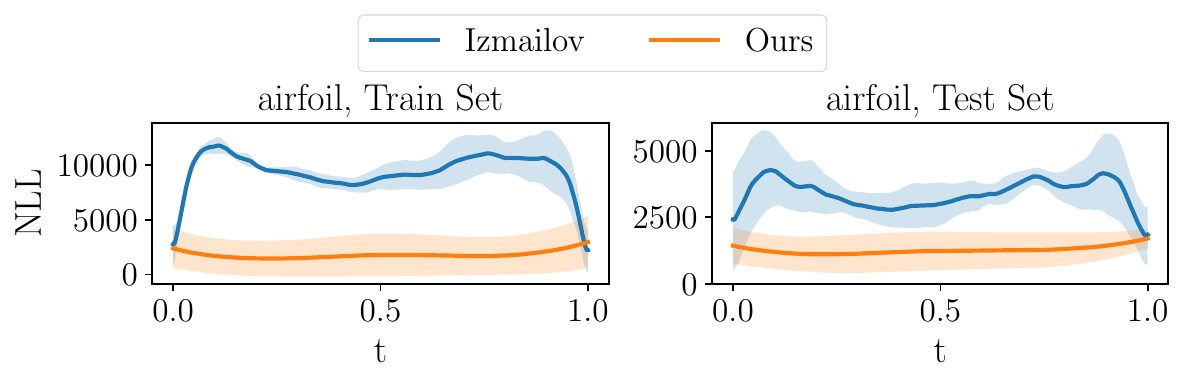}
    \vspace*{-0.5cm}
    \caption{Negative log-likelihood (NLL; y-axis) along the path, comparing the method proposed in \citet{izmailov2020subspace} to our approach across five different data splits and initializations (shaded regions represent one std.dev.).}
    \label{fig:dold_vs_izmailov}
\end{figure}

\textbf{Integrating over} $\mathcal{S}$\, As discussed in the previous section, using  \cref{eq:lossfunc} does not account for the ``speed'' of $\bez$ when computing its expectation. To directly optimize \cref{eq:lossfunc_equal}, we suggest using importance sampling by first drawing $t\sim U(0,1)$ and weighting the sample by $\|\bez_{\paramset}'(t)\| \cdot \pathlen^{-1}$, i.e., the speed at position $t$ divided by the total curve length, computed via numerical integration. Both terms are differentiable in this case. The resulting weight is again a scalar value multiplied with the gradient $\nabla_{\paramset} L$ as in the model updates. Since this introduces a computational overhead (the curve length $\pathlen$ of $\bez$ must be recomputed in every update step), we investigate this adaption of \cref{alg:subspace_cons} in \cref{sec:exp_mnist}.
\begin{algorithm}[ht]
\small
\caption{\small Path Finding}
\label{alg:subspace_cons}
\begin{algorithmic}
    \STATE \textbf{initialize} parameters $\paramset = \{\w_0, \dots, \w_{\K}\}$ randomly
    \WHILE{validation loss still decreasing}
        \FOR {each batch $\mathcal{B}$ of training data $\mathcal{D}_{\text{train}}$}
            \STATE Sample $t^\ast \sim U(0, 1)$
            \STATE Compute $\varw = \bez_{\paramset}(t^\ast)$
            \STATE Compute $\loss(\varw)$ and gradients $\nabla_{\paramset} \mathcal{L}$.
%            \STATE Compute $\nabla_{\sigma}\mathcal{L}$.
            \STATE Update $\paramset$ using SGD
        \ENDFOR
    \ENDWHILE
\STATE \textbf{return} {$\paramset$}
%\STATE Identify orthogonal directions to construct tunnel...
\end{algorithmic}
\end{algorithm}

\subsection{Path Characteristics and Dynamics} \label{sec:proppath}

For simple problems and exact optimization procedures, an optimal solution would be a collapsed curve with length $\pathlen = 0$ located at the global minimum of the loss landscape. Due to the complexity of the loss landscape and the stochastic optimization, we cannot expect the loss $\loss(\bez(t))$ to be constant along the path. %However, we must take the dynamics imposed by SGD into account.
%%%
 %For optimization with SDG decent of simple loss function $\mathcal{L}(\w)$ of the weight $w$ it recently has been proven that the long-term probability $p_i$ of visiting a region $i$ with local minimal loss follows Boltzmann distribution $p_i \sim exp(-E_i/T)$, where the Energy $E_i$ is linked to $\mathcal{L}(\bm{w_i})$ and the temperature $T$ is given by the learning rate \cite{TODO}.  
%%% 
For optimization using SGD on a simple loss function $\loss(\w)$ over weights \( \w \), it has recently been shown that the long-term probability \( p_i \) of visiting a region \( i \) with a local minimal loss follows a Boltzmann distribution \(
p_i \propto \exp\left( -\frac{E_i}{T} \right),
\) where the energy term \( E_i \) is related to the loss \( \loss(\w_i) \), and the temperature \( T \) is related to the learning rate of SGD \citep{azizian2024longrundistributionstochasticgradient}. %stephan2017stochastic}. 
Boltzmann statistics describe a system where energetic terms (loss minimization) and entropic contributions (exploration due to noise) balance to form a stationary distribution.
While the dynamics of \cref{alg:subspace_cons} are more complex compared to standard deep learning optimization, we still expect that the dynamics of the pathfinding result in two competing effects: an energetic term seeking the minimum of the loss functional \( L(\bm{\Lambda}) \) in \cref{eq:lossfunc}, and an entropic term favoring typical configurations while disfavoring non-typical configurations such as entirely straight, elongated, or collapsed paths. This balance between energy minimization and entropy maximization prevents the path from collapsing to a single point and encourages the exploration of a diverse set of low-loss configurations along the path.

\textbf{Simplified Entropic Model}\, To investigate the effect of the entropic contributions, we consider the case in which the path lies in an infinite region of constant loss so that the gradient of $\loss(\varw)$ vanishes. At first glance, this may seem like an overly strong assumption. However, in the long-time limit and without interventions during training, this assumption becomes reasonable if the loss landscape consists of minima surrounded by volumes of low-loss regions. 
%Furthermore, this behavior was empirically observable in the magnitude of the gradient $\sum_{k=0}^K||\nabla {\partial L}/{\partial \w_k}||_2$ (cf. \cref{sec:exp_scaling_behave}). 
Thus, to mimic the stochasticity in training, we replace the gradient $\nabla_{\paramset} L$ with a random component $\bm{\epsilon} \sim \mathcal{N}(0, \noise^2 I_\D)$ in the update of control point $\w_{\kk}$:
%\begin{equation} \label{eq:scaling}
$\w_{\kk} \leftarrow  \w_{\kk} + \bm{\epsilon} \, \eta \omega_k(t^\ast)$,
%\end{equation}
where $\eta$ is the learning rate. This dynamic induces a diffusion process in which the relevant quantities exhibit a characteristic square-root dependence on the effective time constant, $\sqrt{\n \, \eta^2 \noise^2}$, after  $\n$  update steps. 
This model is analogous to the dynamic of a polymer chain where the control points $\w_{\kk}$ represent the monomers or beads of the chain. The center of mass of the chain relative to its starting position diffuses as $\| \pcom \| = \sqrt{\n \, \eta^2 \noise^2} \cdot \sqrt{D}/(K+1)$, as shown by the dotted line in Panel (A) of \cref{fig:scaling_law_plot1}. Also displayed are simulation results for $\D=55$, selected as an intermediate value between $\K=10$ and the maximum $\K=80$ 
%to mimic a lower dimensional loss manifold, 
and different values of 
%$\K$ and 
$\noise$ averaged over 100 repetitions.
%While the spread of the control points w.r.t. the center of mass, a quantity known as the radius of gyration $\rg$ is for $n \ggt \eta^2 \noise^2$
While $\| \pcom \|$ captures the overall movement of the chain, Panel (B) shows the end-to-end distance $\re := \norm{\w_K - \w_0}$, which describes the relative positions between the monomers. 
%shows the total length $\lK = \sum_{\kk = 0}^{\K - 1} \| \w_{\kk + 1} - \w_{\kk} \|$ of the chain of control points given by the Euclidean distance between the $\K+1$ control points and averaged over 100 random initialization.
%
\begin{figure}[ht]
%% Plots created with Simulation_Plotter.R
    \centering
    \includegraphics[width=1.\linewidth]{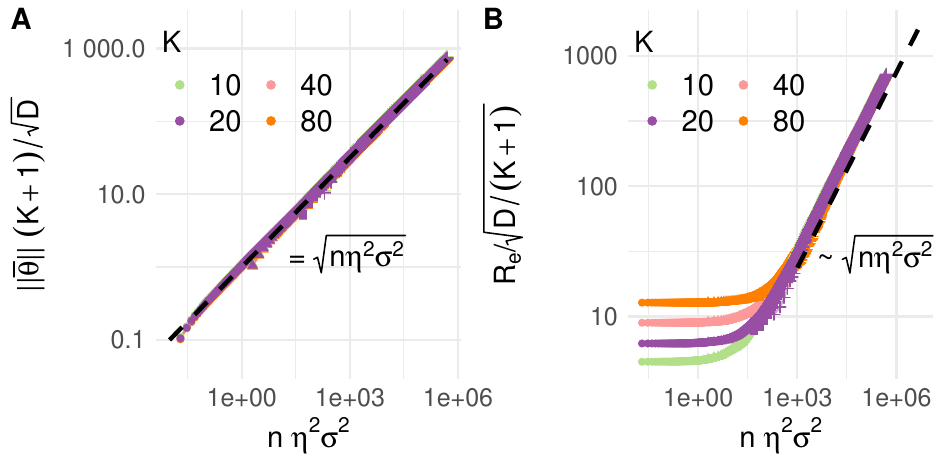}  % Replace with your actual image file path
    \caption{Scaling behavior of path characteristics. Panel (A) shows the center of mass \(\| \pcom \|\), and Panel (B) shows the end-to-end distance \(\re\), both following square-root scaling with the effective time constant. Symbols denote different values of $\noise$: \(\circ \, \sigma = 0.1, \, \triangle \, \sigma = 1, \, \square \, \sigma = 5, \, \lozenge \, \sigma = 10\), though some values are only partially visible due to overplotting.}
    \label{fig:scaling_law_plot1}
\end{figure}
% \paragraph{Path Length} One solution for collapsed path and one solution for maximum path, but many for long(er?) paths
%
After a transient period, $\re$ is proportional to $\sqrt{\n \; \eta^2 \noise^2 \cdot} \sqrt{\D/(K+1)}$. Similar scaling laws are observed for quantities like
%for the norm $\| \pcom \|$ of the center of mass $\pcom := \frac{1}{\K + 1} \sum \w_{\kk}$, 
the squared radius of gyration $\rg^2 = \frac{1}{\K+1} \sum_{\kk = 0}^{\K} \| \w_{\kk} - \pcom \|^2$, $\lK = \sum_{\kk = 0}^{\K - 1} \| \w_{\kk + 1} - \w_{\kk} \|$
, and other characteristic quantities describing the chain (see \cref{sup:scaling} for more details).
% and the end-to-end distance $\re = \norm{\w_K - \w_0}$.
%While the time constant again scales like $\sqrt{\eta \; \n \; \noise^2}$ from the very beginning, the asymptotic dependence on $\K$ for $\| \pcom \|$ is $\| \pcom \| \propto \K^{-1/2}$, and largely independent of $\K$ for $\rg$ (see \cref{sup:scaling} for additional details).
%The scaling behavior of these quantities is dictated by the time constant $\sqrt{\n \cdot \eta^2 \noise^2}$; they differ only in their dependency on simple proportionalities of $\D$, $\K$, or their square roots. 
%Further details are provided in  
%Scaling laws in physical sciences are robust and often hold regardless of system specifics. 
%Scaling laws in physics are often universal, holding across diverse systems.
In polymer physics, characteristic lengths such as $\re$, $\rg$ and $\| \pcom \|$ exhibit universal scaling laws across various conditions, from dense melts to dilute solutions, as well as in simplified models \citep{de1979scaling}. In \cref{sec:exp_scaling_behave}, we will investigate whether this also holds in neural networks. 

\subsection{Tunnel Embedding and Description}
\label{sec:tunnel_emb_lifting}

Having found a path of low loss, we now extend the path to a tunnel as described in \cref{def:tunnel}. This has various advantages, including better uncertainty quantification as discussed in \cref{sec:subspaceinf}. 
%In principle, there are different sensible approaches to ``lift'' the path into the product, such as drawing on the idea of Stochastic Weight Averaging \citep[SWA;][]{izmailov2018averaging} and identifying the directions of highest variance of the training trajectory. 

\textbf{Volume Lifting}\,
In principle, there are several options to lift the path into a higher-dimensional space. 
%surrounding volume. \cite{izmailov2020subspace} (with $\K=2$) or \cite{dold2024bayesian} (with $\K>=2)$ 
For example, \cite{izmailov2020subspace, dold2024bayesian} assume that every direction that the curve is taking reveals valuable insights. These directions are encoded by the subspace $\subspace = \operatorname{span}(\paramset)$. By defining $\subspace$ this way, we discard the information of $\bez$ and instead work with the hyperplane in which $\bez$ resides. To perform inference or investigate model hypotheses, we therefore do not travel along $\mathcal{T}$, but instead define a projection matrix $\proj\in\mathbb{R}^{\D \times \K}$ mapping $\proj: \subspace \to \W$, i.e., taking an element $\sselem\in\subspace$ in this hyperplane (irrespective of $\T$) and map it back to a neural network weight $\w\in\W$.

\textbf{Tunnel Lifting}\,
We propose tunnel over volume lifting because only a small portion of the volume in $\operatorname{span}(\paramset)$ has a low loss. %Thus the sampler needs to find the low-loss tunnel, defined in \cref{def:tunnel}, which can slow down sampling performance (See todo{exp}). 
Instead, we construct a tunnel according to \cref{def:tunnel} that preserves the time information of the original path when lifting it into a higher-dimensional space. %an additional mapping where we leverage the curve properties to focus on larger low-loss regions. 
More specifically, at each time point $t\in\T$, %While the curve traverses the low-loss valley, 
we can uniquely define the first dimension of a local orthogonal basis system by the tangent vector $\tangent(t) = \frac{d\bez(t)}{{d}t}$. The $\K-1$ other dimensions of the tunnel can be constructed by taking orthogonal directions $\tortho_k \in \subspace$ at each position $t$. Together, this defines a local orthogonal basis system $\orthosys(t) := \{\tangent(t), \tortho_1(t), \ldots, \tortho_{\K-1}(t)\}$ and the rank-(K$-$1) vector bundle $\btunsp = \bigcup_{t\in\T} \operatorname{span} \left\{\tortho_1(t),\ldots,\tortho_{\K-1}(t)\right\}$ defines the ambient space of the curve $\bez$. % (i.e., $\T \times \btunsp$ defines the tunnel space).
%where all $\tortho$ directions are used to map a point $t, \btun$ back into the $\subspace$ space. 

%\paragraph{Tunnel Lifting} Since the loss path learning involves optimizing the $\K+1$ control points $\w_{\kk}$, the linear combination of the vectors defining $\bez_{\paramset}$ will yield an ensemble of low-loss models. Traversing across this path, i.e.\ in dimension $\T$, hence corresponds to different weighting and averaging the models defined by $\paramset$. A natural lifting is thus to consider the orthogonal directions to this path, focusing on the $\K-1$ dimensions that change the functional variability of $f$ the most. As these directions are not unique, we need to define a meaningful reference system. More specifically, we can define the tangent vector $\frac{d\bez(t)}{{d}t}$, which defines the local curve direction at each point $t\in\T$. To characterize the ambient space of $\bez$, we can construct orthogonal directions to this tangent vector by, e.g., using Gram-Schmidt orthogonalization or by constructing a local orthogonal basis from the ambient space. Using the resulting $\K$ orthogonal vectors $\projvec_1,\ldots,\projvec_\K$, we can define a projection matrix $\proj\in\mathbb{R}^{\D \times \K}$ with columns $\projvec_k$ defining the mapping $\proj: \subspace \to \W$, i.e., taking an element $\sselem\in\subspace$ in the tunnel and mapping it back to a neural network weight $\w\in\W$. When concatenating this map with $\ell$, we obtain a loss tunnel as described in \cref{def:tunnel}.

\textbf{Rotation Minimizing Frames}\, A natural way to construct the local orthogonal system $\orthosys$ is to use the Frenet–Serret frame \citep{frenet1852courbes, serret1851quelques}, where $\tangent(t)$ is the tangent vector, and subsequent normal, binormal and consecutive vectors are derived from $\tangent'(t), \tangent''(t), \dots, \tangent^{(\K-1)}(t)$. These derivatives are orthogonalized via a Gram-Schmidt process. This approach, however, has several drawbacks. First, the local coordinate system is undefined at any point $t$ where a derivative of $\tangent(t)$ is zero. Second, the larger $\K$, the more expensive this construction will be due to the higher-order derivatives. Lastly, a sign flip in any curvature can cause the orthogonal system to flip, which is problematic for sampling routines presented later, relying on smooth transitions between different states.
\begin{figure}[h]
    \centering
    \includegraphics[width=0.9\columnwidth]{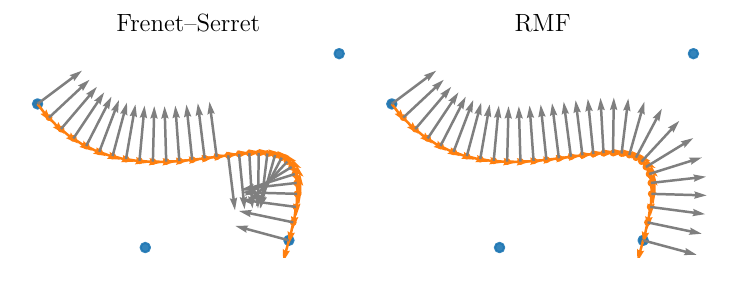}
    \caption{Frenet–Serret Frame (left) and our RMF implementation (right) on an exemplary 2D B\'{e}zier curve.}
    % jax_bezier_grad.ipynb
    %using data fromcite{farouki2016rational}.}
    \label{fig:frenet_rmf}
\end{figure}
To avoid these issues, we instead use a Rotation Minimization Frame algorithm \citep[RMF; see, e.g.,][]{wang2008computation}. RMF's key idea is to minimize the rotational drift while maintaining a consistent frame orientation along the curve (cf.~\cref{fig:frenet_rmf}). For this, we construct a lookup table that starts with the initial orientation of the Frenet-Serret solution at $t=0$. The idea is to propagate this frame $\orthosys(0)$ without the tangent vector $\tangent$ along the curve and orthogonalize the frame at $t^\ast$ through the Gram-Schmidt process with the new tangent vector $\tangent(t^\ast)$ and the previously stored frame. For this, we discretize $\T$ into a large number of discrete time points $0,\ldots, t_i,\ldots, 1$ and iterate through these until we find a point $t_\text{cut}$ where $\sphericalangle( \tangent(0), \tangent(t_\text{cut}))$ exceeds $45^\circ$. This implies that a sign flip can occur for any of the $\tortho_k(t_\text{cut})$ directions. Hence, we add a new reference frame to our lookup table. % whenever $\tangent(t_i) \cdot \tangent(t_\text{cut})$ exceeds some threshold $\epsilon$, with $0^\circ < \epsilon < 45^\circ$. 
This process is repeated until we reach $t=1$.
For given $t^\ast$, we can then create the ambient space of the path by first finding $i:t_i \leq t^\ast < t_{i+1}$, and then initialize the Gram-Schmidt process using the tangent vector $\tangent(t^\ast)$ and the previously stored frame.

\subsection{Tunnel Symmetries}

%Finally, as for the loss paths, we also study the properties of the proposed loss tunnel. 
The idea of spanning a tunnel by orthogonal directions is also desirable from a permutation invariance point of view. More specifically, we can avoid permutation symmetries, which are known to dramatically increase the number of low-loss areas in neural landscapes and hinder mode connectivity \citep{pittorino2022deep}. While the control points themselves are $\mathbb{P}$-almost surely permutation symmetry-free, the path might still contain permutations and thus lead to reduced functional diversity in $f_{\w}$.

\begin{theorem}[informal]\label{thm:perm-inv-informal}
If each $\w\in\paramset$ is permutation symmetry-free, an $\epsilon$-tube, $\epsilon>0$, exists around the path $\bez_{\paramset}$ that also contains no permutation symmetries.
\end{theorem}

In other words, if the optimized path $\bez$ is permutation symmetry-free, then the constructed tunnel will also contain no permutation symmetries in the path's vicinity. A formal statement and proof are given in \cref{app:theory}. To ensure that models in $\paramset$ are permutation-free in each layer, one option is to sort the biases in each layer \citep{pourzanjani_2017_ImprovingIdentifiabilityb}. In our experiments, we observe that even without this explicit enforcement, functional diversity along the path is usually obtained and the optimization of $\bez$ is not prone to get stuck in permutation invariances (cf.~\cref{app:permutation_symmetry}). 
The nearly identical performances of the subspace approach with and without bias sorting in our further experiments suggest that the path itself, without explicit adjustments, consists of permutation-free solutions (cf.\ \cref{app:permutation_symmetry}).

\section{ADVANCING SUBSPACE INFERENCE} \label{sec:subspaceinf}

After optimizing $\bez$ as in \cref{alg:subspace_cons}, sampling values from this curve $\bez$ will likely yield well-performing models since these define linear combinations of the $\K+1$ optimized models $\w_k$. %Traversing through the multidimensional tunnel along this path ideally also better captures the parameters' uncertainty in the near surroundings of this path. 
The constructed tunnels, in contrast, allow to also traverse to directions orthogonal to the curve and better explore the variability in $\w_k$. 

\textbf{Subspace Inference}\, The first to combine ideas of subspace construction and uncertainty quantification were \cite{izmailov2020subspace}. Building on the work of \cite{garipov2018loss}, \cite{izmailov2020subspace} proposed to use three models of a parameterized curve to span a plane $\subspace \subseteq \mathbb{R}^2$ (i.e., do a volume lifting) and run MCMC-based approaches in this subspace. Analogously to this idea, we study the use of the previously proposed tunnels to guide MCMC-samplers through the loss landscape and the choice of a meaningful prior.

\begin{figure}[h]
    \centering
    % \begin{subfigure}[b]{0.49\columnwidth} % Adjusting to make 0.95 total width
    %     \centering
    %     \includegraphics[width=\columnwidth]{figures/prior_varphi.jpg} % Replace with your image
    % \end{subfigure}
    % \begin{subfigure}[b]{0.49\columnwidth}
    %     \centering
    %     \includegraphics[width=\columnwidth]{figures/prior_varphi_adjust.jpg} % Replace with your image
    % \end{subfigure}
    \includegraphics[width=\columnwidth]{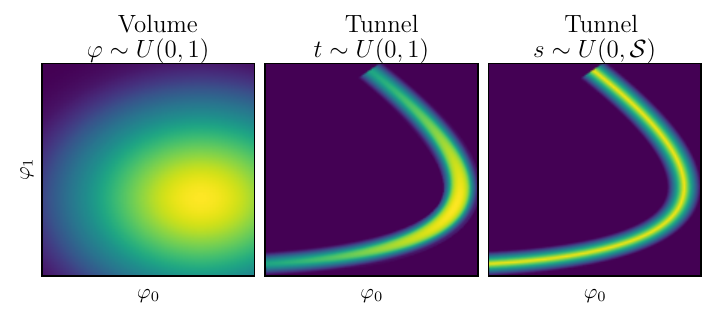}
    \caption{Comparison of a prior in the ``volume space'' (left), in the tunnel space with a uniform prior on $t$ (center) and in the tunnel space with adjusted prior through $s\sim U(0,\pathlen)$ (right). %The visualized underlying B\'{e}zier curve was conducted from the path optimization on the airfoil data set
    }
    \label{fig:prior_comparison}
\end{figure}
%
%
%
%Prior specification for low-loss tunnels is somewhat difficult. 
%For practical reasons, one could employ 

% \newpage

\textbf{Tunnel Priors}\, The most common prior assumption in Bayesian deep learning is an isotropic standard normal prior. While \cite{izmailov2020subspace} argue that the prior is not crucial for good performance if ``sufficiently diffuse'', this does not hold in the aforementioned tunnels. More specifically, while using $\bm{\varphi} \sim \mathcal{N}(0, \bm{I}_{\K})$ over $\subspace$ might yield reasonable results for an unstructured subspace, this prior ignores the embedded structures of $\subspace$ in the case its construction follows the previous tunnel approach. 
%As our projection matrix $\bm{\Pi}$ that projects the subspace parameter into the full-dimensional weight space is composed from the first $\K$ singular vectors of the covariance matrix of the stacked trained curve point parameters scaled by their corresponding singular values, we expect the scale of the prior to be appropriate and not too restrictive for obtaining a sensible posterior over $\bm{\varphi}$. 
%However, defining a prior for sampling directly in the subspace spanned by the orthogonal directions, which we call the $\bm{\varphi}$-space, does not account for the intricate shape of the loss landscape along the B\'{e}zier path. 
Ideally, we would want high prior probability for areas close to the center loss path spanning the tunnel, %where the loss is lowest, with 
and gradually decreasing probability density towards the ``boundaries'' of the tunnel (cf.~\cref{fig:prior_comparison}).\\ 
\underline{Tunnel prior in $t$}: Using prior $t\sim U(0,1)$ and $\btun\sim N(0,\sigma^2\bm{I}_\D)$ solves this problem. However, the tunnel prior with uniform distribution in $t$ suffers from being stretched or squeezed in $\subspace$ due to the varying traversal speed $\| \tangent(t) \|$. Additionally, bendings along the path increase the prior density inside and decrease it outside the curve (center image in ~\cref{fig:prior_comparison}).\\ \underline{Tunnel prior in $s$}: To correct for this effect, we adjust the prior by incorporating the volume change with the Jacobian determinant between $\subspace$ and $(\mathcal{T} \times \btunsp)$ with $\log | \det({\partial g(t,\btun)} / {\partial t,\btun}) |$, where $g \colon (\mathcal{T} \times \btunsp) \to \subspace$ is the mapping from the tunnel into the subspace.
%    
% \begin{equation}
% \label{eq:adaped_prior}
% $p_\varphi(\varphi) \propto p_\lambda(\lambda) \left| \det \left( \frac{\partial f(\lambda)}{\partial \lambda} \right) \right|$
% \end{equation}
While the adjusted prior is only known up to a normalization constant, this is sufficient for the subsequent sampling procedure.
To further reduce the computational complexity of the Jacobian determinant associated with increasing $K$, we only adjust for the different speeds of $\| \tangent(t) \|$ through sampling $s \sim U(0, \mathcal{S})$ instead of $t\sim U(0,1)$ as in the right image of \cref{fig:prior_comparison}. 

% sampling weights to adjust for the stretching or squeezing of the uniform distribution %, we incorporate the following factor into the log-likelihood function during sampling:
% %\begin{equation} \label{eq:weight_sampling}
% $\text{log}\left(\|\bez_{\paramset}'(t)\|\right) - \text{log}\left( \int_0^1 \| \bez_{\paramset}'(t) \|dt \right)$.
% %\end{equation} 
% These weights are used as a multiplicative factor to scale the log-likelihood. Since the required integral only needs to be computed once for sampling, this leaves only $\|\bez_{\paramset}'(t)\|$ to be evaluated for each proposed sample. The full Algorithm is provided in \cref{app:samplalg}.

%Usually, one would apply a Jacobian volume correction to applicable prior distributions defined in the $(t, q)$-space \citep{dold2024bayesian}, where $t$ defines the point along the curve and $q$ the orthogonal distance from the curve at that point. This way, one could integrate the notion of closeness to the center of the valley quite easily into the sampling process, for instance, with a normal prior on $q$ and a uniform prior on $t$. To make sure that such a prior is proper, i.e., that it integrates to 1, we propose the following adjustments: ...

% \subsection{Scalability}

% [Claim that our method is more scalable --- is it? ]

%\subsection{Implementation}

\section{NUMERICAL EXPERIMENTS}

We now empirically investigate loss paths and tunnels, and validate previous theoretical findings. For fundamental properties, we generate synthetic data to ensure a controlled environment. Performance results are obtained on common benchmark datasets. For further results and details see \cref{app:exp_details,app:furthres}.

\subsection{Scaling Behavior in Loss Landscapes} \label{sec:exp_scaling_behave}
Next, we investigate whether there are phases in the complex dynamics of the path (cf.~\cref{sec:proppath}) where the purely entropic nature of the diffusion process is preserved or if energetic terms become dominant. To this end, we use the scaling laws implied by our previous model assumption as a diagnostic tool. 
\begin{figure}
\centering
\includegraphics[width=1.\linewidth]{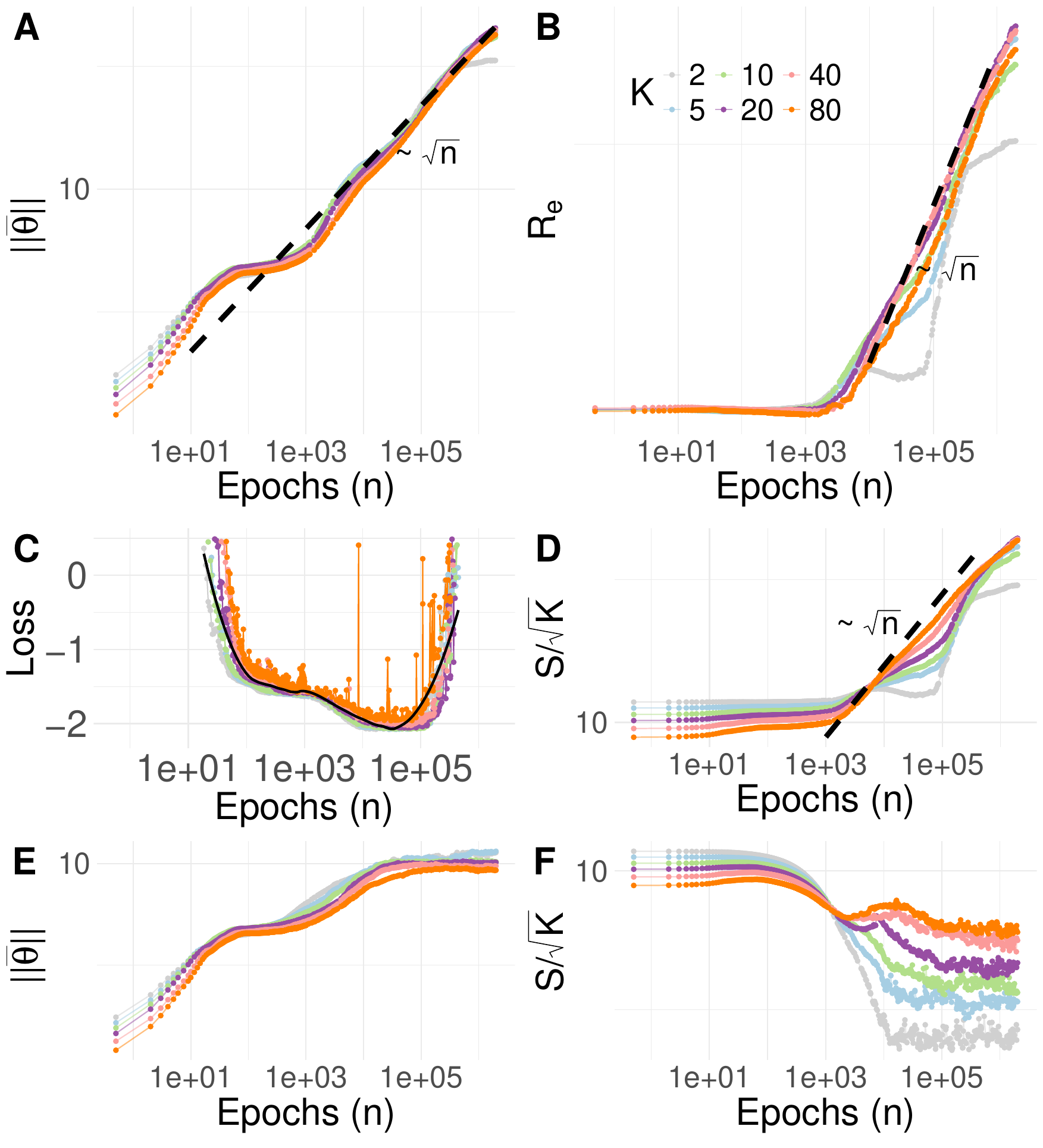}
\caption{Corresponding to \cref{fig:scaling_law_plot1}, Panels A and B show \(\|\pcom\|\) and the end-to-end distance (\(\re\)), respectively. Panel C presents the training loss function, while Panel D depicts the B\'{e}zier curve length (\(\pathlen\)). The lower panels include a weight decay of 0.1, which restricts free diffusion.}

\label{fig:length_scaling}
\end{figure}

\textbf{Results}\, In \cref{fig:length_scaling} we summarize our experiment.
%two typical quantities are depicted that dynamically change throughout the diffusion process. 
%Panel A and B corresponds the same Panels in  upper half, Panel A shows the movement of the entire chain ($\|\pcom\|$), while Panel B illustrates the dynamics of the end-to-end distance ($\re$). 
%Panel C displays the training loss function, and Panel D shows the total length of the Bézier curve ($\pathlen$). Alongside these, the temporal dependency of our simple diffusion model is also illustrated.
%
We expect to observe an initial drift phase where the entire chain moves within the loss landscape toward a minimal region. Looking at the loss function in Panel C, we observe a plateau reached after approximately $500$ epochs. This behavior is also reflected in the magnitude of the gradient $\sum_{k=0}^K||\nabla {\partial L}/{\partial \w_k}||_2$ (cf.~\cref{fig:length_scaling_gradient_norm}).
From around 2000 epochs onward, we see the typical behavior of $\sqrt{n}$-growth in length quantities B) and D), which corresponds to the free diffusion model, as predicted in Section~\ref{sec:proppath}.
%While the scaling behavior with respect to the number of control points $\K$ differs between the simplified model and the observed results---exhibiting distinct trends for both SGD and Adam (see \cref{sup:scaling_exp})---these differences are not critical and can be attributed to the model’s assumption of isotropic Gaussian noise. More importantly, the scaling with respect to the time constant (here, $\propto \sqrt{\text{epochs}}$) remains consistent across all settings.
%In \cref{fig:length_scaling}, we demonstrate that the scalings of three length quantities ($\lK$, $\re$, and $\pathlen$, estimated via numerical integration) are all $\propto \sqrt{\text{epochs}}$. This result suggests that, despite variations in $\K$-scaling, the fundamental behavior of unhindered diffusion in the simplified model remains relevant in the more complex dynamics of path optimization.
%
As shown in Panel D, the curve length $\pathlen$ also exhibits this dependency. This is consistent with the relationship $\re \leq \pathlen \leq \lK$. By analogy to polymers, we expect further that all quantities related to the distances between monomers share a similar diffusion time constant, which we also observed for other measures like $\rg$ (see \cref{sup:scaling}). 

From around $10^5$ epochs, we observe instabilities in the optimization process, while the scaling laws still hold. Thus, it is recommended to stop training before reaching this regime. In our applications, these instabilities were not problematic, as early stopping based on validation data consistently halted training before reaching this point. %to use early stopping.  
%This dynamic is not yet fully understood; it may well be due to numerical artifacts or because we are already very close to a minimum, where any fluctuation leads to a sudden increase in the loss.
%
%
The free diffusion behavior changes when external potentials are applied, as these introduce forces that 
%, while seemingly subtle, 
can significantly alter the dynamics of the system. Weight decay, for example, can be viewed as adding a quadratic penalty term \(\frac{\lambda}{2} \|\w\|^2\) to the loss function, effectively creating a harmonic potential that pulls weights towards zero. This additional potential restricts the free diffusion, and as shown in the lower panels of \cref{fig:length_scaling}, can halt the unrestricted movement of the weights or, in extreme cases, cause the curve length to collapse.

\subsection{Volume vs.~Tunnel Sampling}
\label{sec:exp_tunnelVsVolume}

\begin{figure}
    \centering
    \includegraphics[width=1.\linewidth]{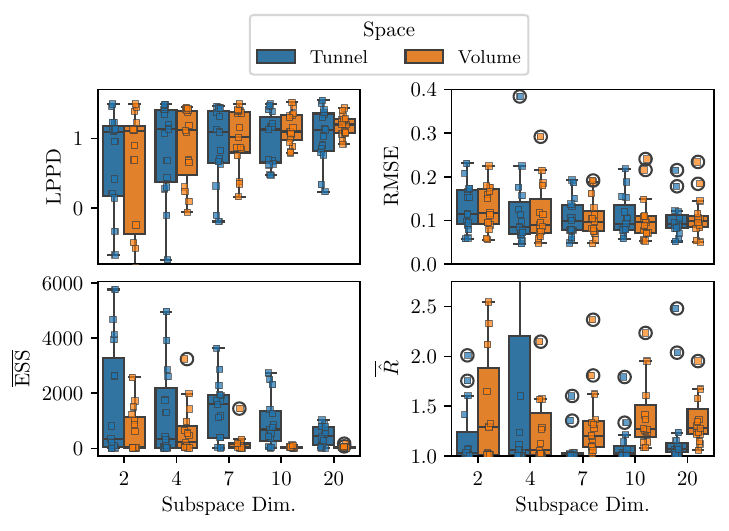}
    \caption{Performance (upper row) and sampling (bottom row) metrics on the synthetic dataset over 15 repetitions. %The temperature parameter was selected by optimizing the validation data. 
    Small square points show the underlying data. $\text{ESS}$ and ${\hat{R}}$ are averages over all parameters.}
    \label{fig:performance_lambda_phi}
\end{figure}

To investigate the influence of our approach, we perform the tunnel lifting with its tunnel prior on the simulated data (cf.~\cref{app:simdata}) and investigate the performance of the sampling procedure on test data as well as the quality of the obtained samples. We compare this to the volume lifting approach using the log posterior predictive density (LPPD), root mean squared error (RMSE), averaged Effective Sample Size ($\text{ESS}$), and the averaged Gelman-Rubin $\hat{R}$ metric.

\textbf{Results}\, \cref{fig:performance_lambda_phi} visualizes the performance comparison between the two lifting approaches. While there is little change between the two spaces in prediction performance, we see a notable trend in performance improvement for larger subspace dimensions, particularly in the LPPD, and to a lesser extent in the RMSE. As $K$ increases, it also becomes more likely to obtain better generalizing models, since the worst performance across repetitions improves. 
Furthermore, the tunnel approach shows notably better average ESS and $\hat{R}$ values, while there is a clear decline in the number of effective samples when using the volume approach. This supports the hypothesis that our tunnel construction improves the problem's conditioning.

% \begin{figure}
%     \centering
%     \includegraphics[width=1.\linewidth]{figures_dd/reg_best_t_sampling_metric_not_clean.pdf}
%     \caption{Sampling performance over 15 repetitions. The right plot shows the Effective Sample Size (ESS) and the left plot shows the $r\_\text{hat}$ averaged over all parameters.}
%     \label{fig:sampling_lambda_phi}
% \end{figure}

\begin{figure}
    \centering
    \includegraphics[width=\linewidth]{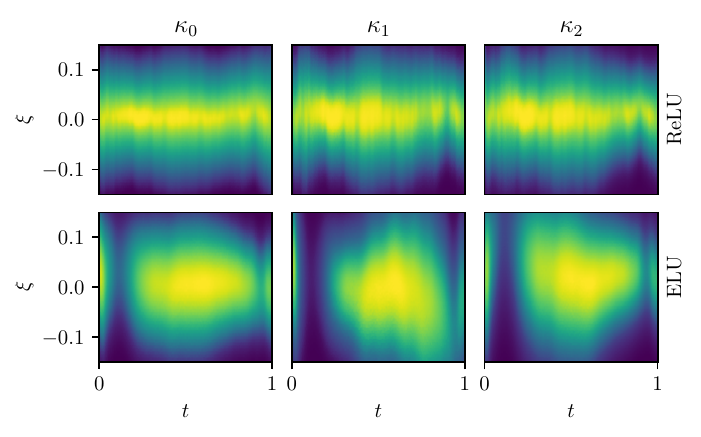}
    \caption{First four dimensions of the tunnel space (x-axis and different columns) for different activation functions (different rows) on the yacht dataset %by traveling through the tube space, where the x-axis represents the $t$ dimension and the y-axis one $\btun$ dimension. 
    with remaining $\btun$-dimensions set to zero. The x-axis represents the $t$-dimension, and the y-axis is one $\btun$-dimension. The color represents the unnormalized log posterior value.  %The upper row shows the results for the first three orthogonal directions using a ReLU activation function. The bottom line shows the same result, except for using an ELU activation function
    }
    \label{fig:smooth}
\end{figure}

\begin{table*}[!h]
    \centering
     \caption{Performance comparisons across different methods (columns) and datasets (rows) for the LPPD and RMSE metric. The method closest to the MCMC gold standard is highlighted in \textbf{bold}. %DE refers to Deep Ensemble \citep{lakshminarayanan}, LA to Laplace Approximation \citep{daxberger2021laplace}, ModeCon to Subspace Inference \citep{izmailov2020subspace}, and Tunnel-$K$ to our approach using tunnel lifting.
    }
    \label{tab:benchmark}
    \resizebox{.85\textwidth}{!}{
    \begin{tabular}{cc|r|rrrrr}
    \multicolumn{2}{c}{} & \multicolumn{1}{c}{MCMC} & \multicolumn{1}{c}{DE} & \multicolumn{1}{c}{LA}  & \multicolumn{1}{c}{ModeCon} & \multicolumn{1}{c}{Tunnel-$2$} & \multicolumn{1}{c}{Tunnel-$20$} \\
    \midrule
    \multirow[c]{5}{*}{\rotatebox[origin=c]{90}{LPPD ($\uparrow$)}} 
    & airfoil     & 0.67 ± 0.04 & -1.09 ± 0.05 & -3.05 ± 0.16 & 0.09 ± 0.05  & 0.25 ± 0.09  & \textbf{0.34} ± 0.04 \\
    & bikesharing & 0.28 ± 0.04 & -0.94 ± 0.03 & -2.83 ± 0.11 & 0.02 ± 0.03 & 0.12 ± 0.04   & \textbf{0.14} ± 0.03 \\
    & concrete    & 0.09 ± 0.18 & -0.94 ± 0.07 & -2.99 ± 0.27 & -0.29 ± 0.14 & -0.08 ± 0.05 & \textbf{-0.04} ± 0.09 \\
    & energy      & 2.18 ± 0.07 & -0.91 ± 0.07 & -2.90 ± 0.19 & 1.59 ± 0.05 & \textbf{1.74} ± 0.14   & 1.70 ± 0.28 \\
    & yacht       & 2.90 ± 0.19 & -0.91 ± 0.14 & -2.91 ± 0.63 & \textbf{1.33} ± 0.80 & 0.75 ± 3.01   & \textbf{1.33} ± 1.75 \\
    \midrule
    \multirow[c]{5}{*}{\rotatebox[origin=c]{90}{RMSE ($\downarrow$)}} 
    & airfoil       & 0.15 ± 0.02 & 0.30 ± 0.02 & 0.27 ± 0.03 & {0.22} ± 0.01  & 0.19 ± 0.01 & \textbf{0.18} ± 0.01	 \\
    & bikesharing   & 0.21 ± 0.01 & 0.26 ± 0.01 & 0.25 ± 0.01 & 0.25 ± 0.01 & \textbf{0.23} ± 0.01	 & \textbf{0.23} ± 0.01	 \\
    & concrete      & 0.25 ± 0.02 & 0.31 ± 0.01 & 0.44 ± 0.03 & 0.31 ± 0.02 & 0.30 ± 0.02	 & \textbf{0.28} ± 0.02	 \\
    & energy        & 0.03 ± 0.00 & 0.07 ± 0.01 & 0.06 ± 0.01 & {0.05} ± 0.00 & \textbf{0.04} ± 0.01	 & 0.05 ± 0.01	 \\
    & yacht         & 0.04 ± 0.01 & 0.13 ± 0.03 & 0.13 ± 0.04 & 0.06 ± 0.01 & 0.05 ± 0.01	 & \textbf{0.04} ± 0.01	 \\
    \bottomrule
    \end{tabular}
    }
\end{table*}

\subsection{Tunnel Smoothness}
\label{sec:exp_tunnelSmoothness}
%
% - In der FFT aus dem gradient entlang t sieht man leider nichts (Mache das vlt auch was flasch) 
% - Bisherige Ergebnisse:
% - - conditional smootheness plots über ein 2d grid
% - - loss entlang t für verschieden activation functions
%
Next, we investigate how the tunnel is evolving across different dimensions and different activation functions. The latter was found to have a notable influence on the subspace's smoothness, potentially affecting the sampling efficiency. For this, we visualize the first four dimensions of the created tunnels together with the unnormalized log posterior values on the yacht dataset. The results (cf.~\cref{fig:smooth}) confirm that the tunnel for ReLU networks is coarser compared to ELU activation, likely due to the non-continuously differentiable nature of ReLU. ELU, on the other hand, exhibits smoother transitions, although the multimodal structure observed in \cref{fig:smooth} can hinder efficient sampling.   %In contrast, with ELU activation, the ambient space more closely resembles multiple smooth, spherical structures.
%confirm that the tunnel for the ReLU networks is coarser than with ELU activation, probably because of the not continuously differentiable property of ReLU versus ELU activation. Whereas using the ELU activation, the ambient space resembles more likely smooth spheres.%, the ReLU network looks more like an actual tunnel.
%

%\subsubsection{Sampling performance}

% \subsubsection{Temperature effects}

\subsection{Regression Benchmarks}
\label{sec:regression_benchmarks}

To analyze the performance of sampling-based tunnel inference, we extend the benchmark setup from \cite{sommer24connecting}, with a Bayesian neural network in a homoscedastic regression setting with three hidden layers, each with 16 neurons and ReLU-activation using an MCMC-based solution for the UCI benchmark datasets {airfoil}, {bikesharing}, {concrete}, {energy}, and {yacht}. We take their \textbf{MCMC} method based on NUTS as the gold standard and compare it with the mode connectivity (\textbf{ModeCon}) method from \cite{izmailov2020subspace}, and our tunnel (\textbf{Tunnel}-$\K$) approach. We further use Laplace approximation \citep[\textbf{LA};][]{daxberger2021laplace} and a non-Bayesian deep ensemble \citep[\textbf{DE};][]{lakshminarayanan} as baselines. We use five data splits to estimate the standard deviation in the performance metrics.

\textbf{Results}\, Based on the results in \cref{tab:benchmark}, our method consistently outperforms both DE and LA in terms of LPPD. Furthermore, we observe improvements over ModeCon, except for the yacht dataset. 
%For this dataset, we encounter a substantially high standard deviation, suggesting that the LPPD computation is quite sensitive to data splits, resulting in increased variability. 
More importantly, increasing our tunnel's dimension $K$ leads to improvements in most cases. This suggests that the loss surface exhibits additional complexity that only a higher-dimensional subspace can capture. % this complexity, thereby improving functional diversity.
A similar trend holds for RMSE values, where we outperform LA, DE, and ModeCon, albeit by a smaller margin. 

\subsection{MNIST} \label{sec:exp_mnist}

To demonstrate that our method is also applicable to other network architectures, we run our tunnel-based approach using our tunnel-prior for the MNIST dataset to investigate whether a more complex model architecture might benefit from a ``larger'' tunnel. We investigate a subspace dimension of $\K \in \{2,5,10,20\}$ and again report LPPD values in addition to the accuracy on MNIST's standard test dataset. By altering the path optimization as discussed in \cref{sec:pathoptim}, we also investigate the difference in optimizing \cref{eq:lossfunc} or \cref{eq:lossfunc_equal}. As a baseline, we use LA.

\textbf{Results}\, \cref{fig:mnist} depicts the results, showing an increase in LPPD for an increase in subspace dimension. However, an increase in dimension is not necessarily beneficial for prediction accuracy (ACC) as the tunnel dimensions expand away from the path of well-performing ensemble members, confirmed by the empirical results. Additionally, there is little to no noticeable difference between the sampling procedures (indicated by different colors) used to approximate the integral over the loss path. Compared to the LA baseline, subspace inference yields notably better performance in terms of LPPD but does not surpass LA in its accuracy. Further analysis suggests a trade-off between optimizing LPPD and accuracy, which can, for example, be adjusted using the temperature parameter (cf.~\cref{app:mnist_temperature}). 
This again highlights the natural property of the tunnel as a way to capture surrounding uncertainty in parameter space while deviating from the path of high-performing models.

\begin{figure}
 % \hspace*{0.5cm}
 \centering
    \includegraphics[width=\linewidth]{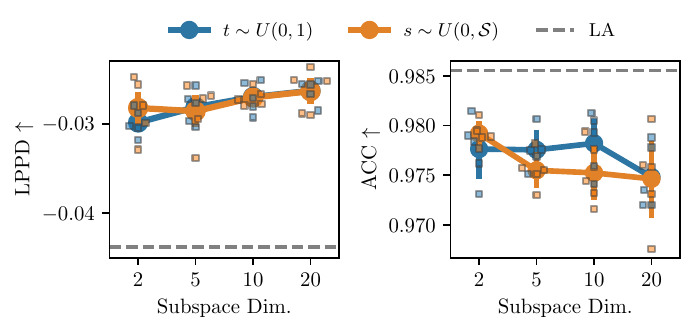}
    \caption{Mean performance on MNIST (with error bars visualizing standard error~over five different folds) using different sampling approaches to optimize the path (color) and subspace dimensions $\K$ (x-axis). The dashed line shows results obtained with Laplace Approximation (LA).}
    \label{fig:mnist}
\end{figure}

\section{CONCLUSION}

In this work, we formalized and analyzed loss paths and loss tunnels in neural network loss landscapes. Drawing on a range of theoretical and empirical perspectives, we offer new insights into these landscapes and discuss various properties relevant when resorting to such approaches. Empirical performance results suggest that the creation of paths and tunnels can be beneficial for sample-based inference.

\textbf{Limitations and Future Work}\, Our study focused on loss tunnels constructed along a pre-defined optimized path. Directly embedding tunnels into neural landscapes by, e.g., pre-defining tunnel properties a priori is a potentially interesting future direction that would allow the direct optimization of such a tunnel. 

%\subsubsection*{Acknowledgements}
%All acknowledgments go at the end of the paper, including thanks to reviewers who gave useful comments, to colleagues who contributed to the ideas, and to funding agencies and corporate sponsors that provided financial support. 
%To preserve the anonymity, please include acknowledgments \emph{only} in the camera-ready papers.
% Hier Beate nicht vergessen und evt. Georg
\subsubsection*{Acknowledgements}
%All acknowledgments go at the end of the paper, including thanks to reviewers who gave useful comments, to colleagues who contributed to the ideas, and to funding agencies and corporate sponsors that provided financial support. 
%To preserve the anonymity, please include acknowledgments \emph{only} in the camera-ready papers.
We thank the reviewers, Beate Sick and Georg Umlauf, for their valuable feedback and fruitful discussions.
This work was supported by the Carl-Zeiss-Stiftung in the project "DeepCarbPlanner" (grant no. P2021-08-007).

% \clearpage

\bibliography{bibliography}

\appendix

\clearpage

\onecolumn

\section*{SUPPLEMENTARY MATERIAL}

\section{Derivation of Scaling Laws and Further Results}
\label{sup:scaling}
In the absence of a loss gradient (i.e., in a flat loss landscape), the updates to the control points \(\w_{\kk}\) are dominated by the noise term \(\bm{\epsilon}\). The change in \(\w_{\kk}\) for a single update step is given by:
\begin{equation}
    \Delta \w_{\kk} = \w^{(n+1)}_{\kk} - \w^{(n)}_{\kk} = \eta\, \bm{\epsilon}\, \binom{\K}{\kk} {t^\ast}^{\kk} (1 - t^\ast)^{\K - \kk}.
    \label{eq:wk_update}
\end{equation}
The dependence on \(\K\) arises due to how the noise in the updates is distributed among the control points of the B\'{e}zier curve. In the update rule, the factors \(\binom{\K}{\kk} t^\ast{}^{\kk} (1 - t^\ast)^{\K - \kk}\) correspond to the weights of a binomial distribution, which are sharply peaked around \(\kk = \K t^\ast\) for large \(\K\). This means that at each time, only a subset of control points near \(\kk = \K t^\ast\) receive significant updates, while others receive negligible updates.

\subsection{Effect on the Center of Mass}
The center of mass of the control points is defined by:
\begin{equation}
    \pcom = \frac{1}{\K + 1} \sum_{\kk = 0}^{\K} \w_{\kk}.
\end{equation}
The change in the center of mass \(\pcom\) from one time step to the next can be expressed as:
\begin{equation}
    \Delta \pcom = \pcom^{(n+1)} - \pcom^{(n)} = \frac{1}{\K + 1} \sum_{\kk = 0}^{\K} \Delta \w_{\kk} = \frac{\eta\, \bm{\epsilon}}{\K + 1} \sum_{\kk = 0}^{\K} \binom{\K}{\kk} {t^\ast}^{\kk} (1 - t^\ast)^{\K - \kk} = \frac{\eta\, \bm{\epsilon}}{\K + 1},
    \label{eq:wk_step}
\end{equation}
since the sum over the binomial coefficients equals 1. If we start with $\pcom^0=0$,
the center of mass after $n$ steps is given by adding up \cref{eq:wk_step} $n$ times, yielding
\begin{equation}
    \pcom^{(n)} =\frac{n \eta}{\K + 1}  \sum_{i=1}^n{\eps}.
    \label{eq:com_n}
\end{equation}
From this random process, we can calculate the expectation and variance. Assuming $\eps \sim \mathcal{N}(0,\noise \mathbf{I}_D)$, the expectation $\mathbb{E}(\pcom_n)=0$ and square root of the variance is:
\begin{equation}
    \|\pcom\| = \sqrt{\var{\pcom}} = \sqrt{n \eta^2 \noise^2} \cdot \frac{\sqrt{D}}{(K+1)}.
    \label{eq:com_n_scale}
\end{equation}
\cref{fig:com_theo} shows the near-perfect agreement between the simulation results of the distance of the center of mass from its origin and the theoretical result (solid lines) of \cref{eq:com_n_scale}. 
\begin{figure}[H]
    \centering
    %Created with com_theo.R
    \includegraphics[width=0.5\linewidth]{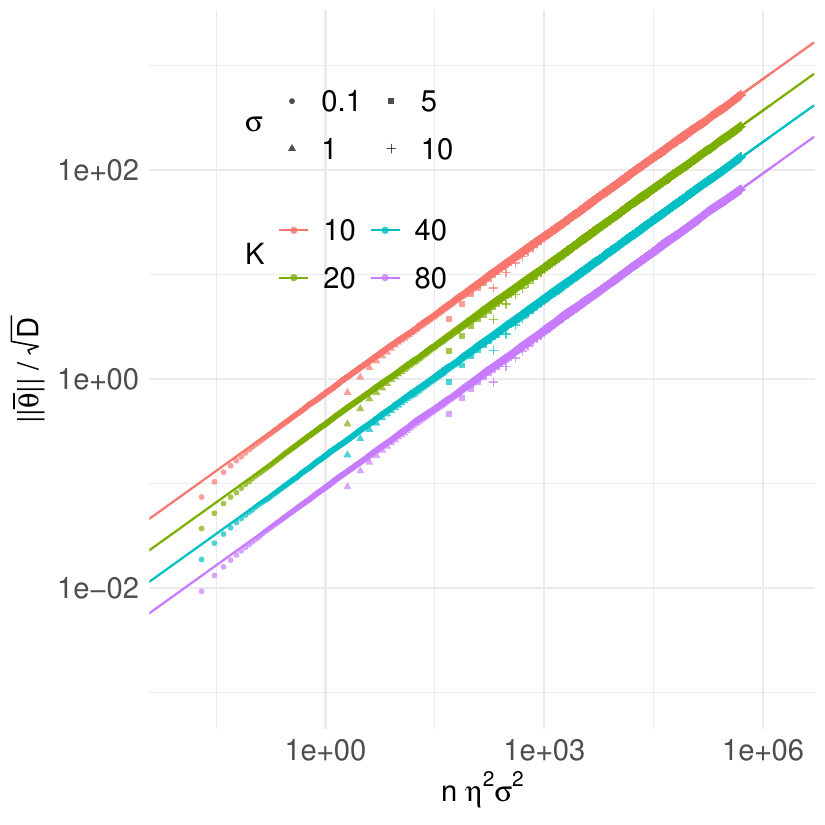}  % Replace with your actual image file path
    \caption{Center of mass comparison of \cref{eq:com_n_scale} with simulation results for various $K$s (color) and different values of added noise $\noise$ (shapes) using starting value $\pcom^{(0)}=0$.}
    \label{fig:com_theo}
\end{figure}

\subsection{Additional Geometric Quantities and Scaling Laws}

To better understand the structure and dynamics of the control points \(\w_k\) in our model, we analyze several geometric quantities analogous to those in polymer physics. These quantities help us quantify how the control points spread and move over time due to stochastic updates. 

\paragraph{Scaling with respect to $n$  and  $D$}
For all geometric quantities discussed, the scaling with respect to the dimensionality $D$ and the number of update steps  $n$  is identical. Specifically, they all scale proportionally to  $\sqrt{n}$  and  $\sqrt{D}$. This is because the updates of the control points are driven by a random walk in a  $D$-dimensional space. Over $n$  steps, the expected displacement in a random walk grows like $\sqrt{n}$ due to the accumulation of variance from independent random steps. Similarly, in $D$  dimensions, each independent component contributes to the total displacement, leading to a scaling of $\sqrt{D}$  for the magnitude of the displacement.

\paragraph{Scaling with Respect to $K$}
To understand the scaling with respect to the number of control points $K$, we consider the average end-to-end distance after $n$ steps:
\[
\re^2 = \mathbb{E}\left[ \left\| \w^{(n)}_{K} - \w^{(n)}_0 \right\|^2 \right].
\]
Using the update rule for the control points (from \cref{eq:wk_update}), the position of control point \( \w^{(n)}_k \) after  $n$  steps is
\[
\w^{(n)}_k = \w^{(0)}_k + \sum_{i=1}^n \Delta \w^{(i)}_k,
\]
where \( \w^{(0)}_k \) is the initial position, and \( \Delta \w^{(i)}_k \) is the update at time step $i$, given by
\[
\Delta \w^{(i)}_k = \eta\, \bm{\epsilon}^{(i)}\, c_k,
\]
with  $c_k = \binom{K}{k} {t^\ast}^{k} (1 - t^\ast)^{K - k}$ . Assuming the initial positions are zero or negligible for long times, the difference between the positions of control points  $k = K$  and  $k = 0$  after  n  steps simplifies to:
\[
\w^{(n)}_{K} - \w^{(n)}_0 = \sum_{i=1}^n \left( \Delta \w^{(i)}_{K} - \Delta \w^{(i)}_0 \right) = \eta\, ( c_{K} - c_0 ) \sum_{i=1}^n \bm{\epsilon}^{(i)}.
\]
To calculate the expectation, we note that there are two independent terms we have to average. First, the random noise $\bm{\epsilon}$ introducing a random walk for which the expectation of $\norm{\sum_{i=1}^n \bm{\epsilon}^{(i)}}^2$ is $n D \sigma^2$. Second, the averaging over the uniform $t^\star \sim U(0,1)$ yields
\[
    \mathbb{E}_t{^\star} \left[\norm{c_{K} - c_0}^2\right] = 
    \int_0^1 (t^K - (1 - t)^K)^2 \; dt \propto \frac{1}{K} \text{for} K \gg 1.
\]
Thus, the expected squared end-to-end distance scales as
\begin{equation}
    \re^2 \propto \eta^2 \sigma^2 n \frac{D}{K+1}.
\end{equation}
\paragraph{Typical Length  $\delta$}
Instead of relying on exact derivations, we employ informal scaling arguments to derive the following quantities, which are then verified numerically. In scaling theory, the typical length represents the characteristic length scale of the problem \citep{de1979scaling}. In our case, the typical length scale $\delta$ of the polymer chain can be formally interpreted as the average distance between the control vectors. 
Viewing the chain of control points as a random walk of $K$ steps, the end-to-end distance relates to the typical length between control points:
\[
\re \propto \sqrt{K} \delta.
\]
Solving for $\delta$ (ignoring potentials differences between $K$ and $K+1$), we get
\[
\delta \propto \frac{ \eta\, \sigma\, \sqrt{ n D } }{ K }.
\]
\paragraph{Total Length \( \lK \)}
The total length of the chain is $K$ times the typical length
\[
\lK = \sum_{k=0}^{K-1} \left\| \w_{k+1} - \w_k \right\| \approx K \delta \propto \eta\, \sigma\, \sqrt{ n D }.
\]
This shows that \( \lK \) is independent of  $K$.

\paragraph{Radius of Gyration \( \rg \)}
The radius of gyration measures the average squared distance of the control points from their center of mass:
   \[
    \rg^2 = \frac{1}{K+1} \sum_{k=0}^{K} \norm{\w_k - \pcom}^2.
    \]
A scaling argument for the radius of gyration is as follows: if we increase the size of the system by a certain factor, the radius of gyration should increase by the same factor. Thus, the radius of gyration should scale like the typical length $\delta$. In a random walk with  $K$  steps, \( \rg \), which measures the spread of the chain’s points around its center of mass, exhibits the same dependence on $K$  as the end-to-end distance \( \re \). Therefore, we have \( \rg \propto \delta \sqrt{K} \). For a more detailed derivation, see \cite{de1979scaling}. This scaling leads to the following expression: 
    \[
    \rg \propto \sqrt{n \, \eta^2 \, \sigma^2} \cdot \sqrt{\frac{D}{\K+1}}.
    \]
Note that, when using scaling laws, we cannot distinguish between $K$ and $K+1$ precisely; however, we choose $K+1$ for better alignment with the observations from the subsequent simulation study. 

\clearpage
\subsection{Simulation Study}

We conducted a comprehensive simulation study to validate the derived scaling laws and analyze the geometric properties of loss paths in neural networks. The study systematically varies three key parameters: the number of control points ($K$), the dimensionality of the space ($D$), and the noise amplitude ($\noise$). Each combination of values is simulated over a large number of update steps $n$ and averaged over 100 repetitions, yielding approximately 60,000 data points. Despite variations in  $K$, $D$,  and  $\noise$, all quantities collapse onto a single curve when plotted against the effective time parameter  $n \eta^2 \sigma^2$ for large values of it. This is shown in \cref{fig:scaling_laws}, confirming the universal nature of the scaling laws, which are summarized in column 'Simplified Model Scaling' in \cref{tab:scaling_comparison}.
\begin{figure}[h!]
    \centering
    \includegraphics[width=0.8\textwidth]{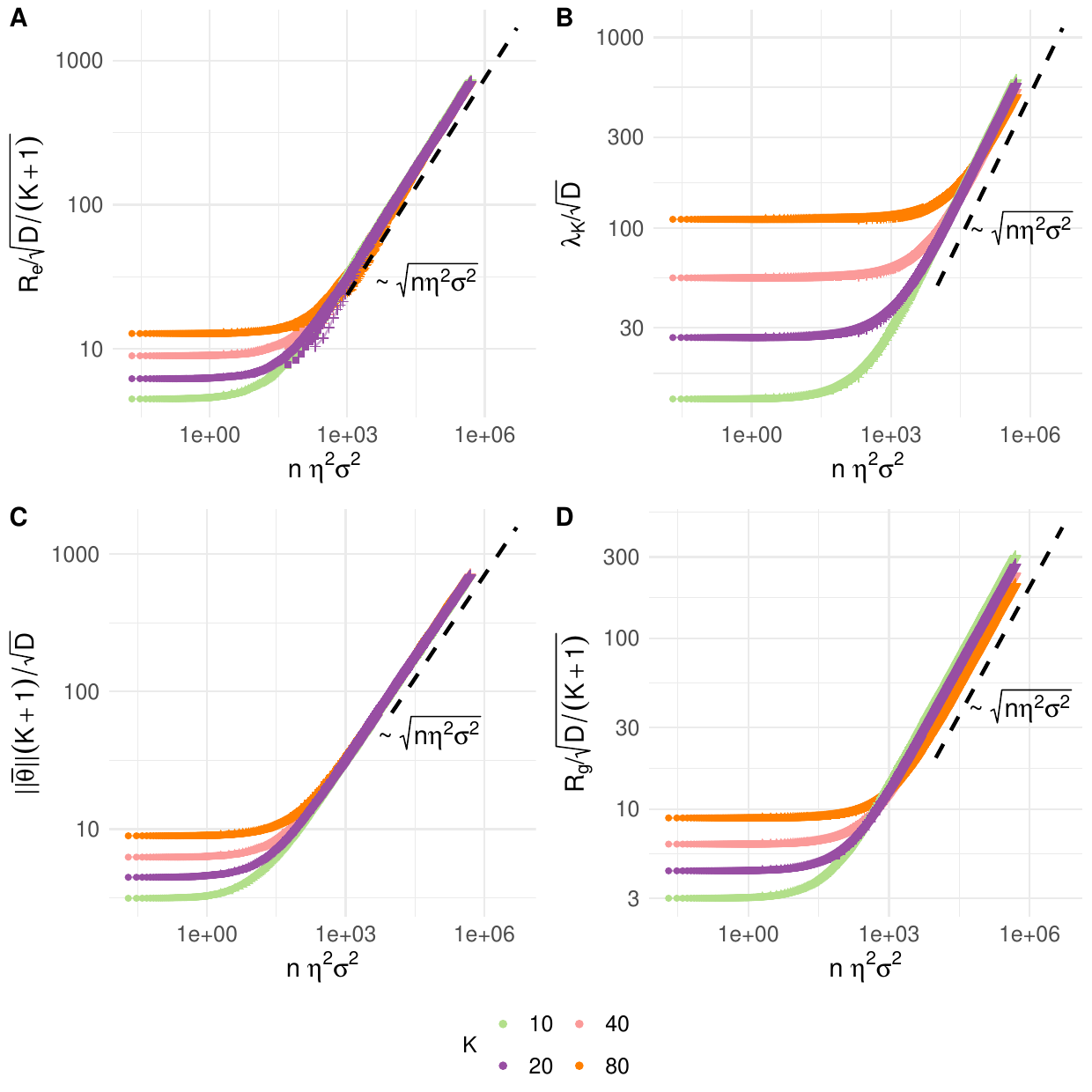} 
    \caption{Simulation study and scaling laws: subplot (A) $\re$, (B) $\lK$, (C) $\norm{\pcom}$ without substracting starting point (D) $\rg$. Quantities in A-D plotted to show scaling behavior.}
    \label{fig:scaling_laws}
\end{figure}

\subsection{Comparison with Path Optimization}
\label{sup:scaling_exp}
This subsection compares our approach to path optimization methods using \cref{alg:subspace_cons}, expanding on the main results. Figure \ref{fig:length_scaling_appendix} illustrates the metrics $\rg$ and $\lK$, which are not included in \cref{fig:length_scaling}. We can observe the scaling behavior with respect to $\sqrt{n}$ for large epochs $n$. Table~\ref{tab:scaling_comparison} summarizes and compares our theoretical findings to the scaling observed in our simplified model.
\begin{figure}
\centering
\includegraphics[width=.4\linewidth]{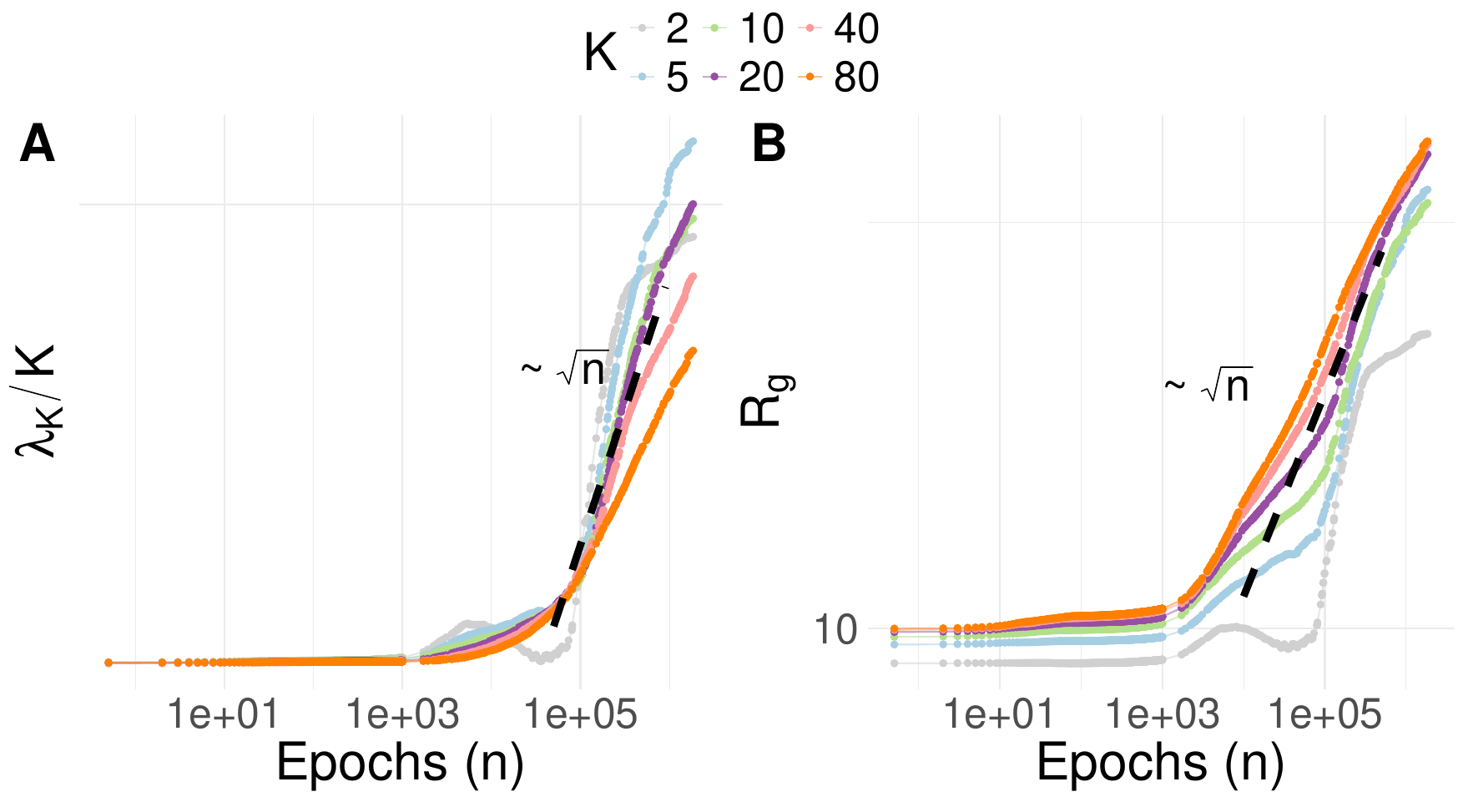}
\caption{Additional results to \cref{fig:length_scaling} for the path optimization with zero weight decay.}
\label{fig:length_scaling_appendix}
\end{figure}
While the scaling with respect to the time constant appears to be approximately $\propto \sqrt{n}$, we observe notable deviations in the scaling with $K$ and $D$. These differences are expected due to the complex dynamics introduced by the Adam optimizer.  
\begin{table}[H]
\centering
\begin{tabular}{lcc}
\toprule
\textbf{Quantity} & \textbf{Simplified Model Scaling} & \textbf{Simulation Scaling} \\
\midrule
Radius of Gyration \( R_g \) &
\( R_g \propto \sqrt{n \, \eta^2 \sigma_n^2} \cdot \sqrt{\frac{D}{K+1}} \) &
\( R_g \propto \sqrt{n} \)\\[2ex]

End-to-End Distance \( R_e \) &
\( R_e \propto \sqrt{n \, \eta^2 \sigma_n^2} \cdot \sqrt{\frac{D}{K+1}} \) &
\( R_e \propto \sqrt{n} \) \\[2ex]

Center of Mass \( \|\bar{\w}\| \) &
\( \|\bar{\w}\| \propto \sqrt{n \, \eta^2 \sigma_n^2} \cdot \frac{\sqrt{D}}{K+1} \) &
\( \|\bar{\w}\| \propto \sqrt{n} \) \\[2ex]

Total Path Length \( \lambda_K \) &
\( \lambda_K \propto \sqrt{n \, \eta^2 \sigma_n^2 D} \) &
\( \lambda_K \propto K \) \\
\bottomrule
\end{tabular}
\caption{Comparison of scaling behaviors between the simplified model and simulations.  }
\label{tab:scaling_comparison}
\end{table}

To further evaluate our assumption that the curve lies in a low-loss valley and that the influence of the objective vanishes, we analyze the gradient norm for all control points during training. In \cref{fig:length_scaling_gradient_norm}, we empirically observe that the sum of the gradient norms of all control points reaches its minimum after approximately 100 epochs, which corresponds to the loss stagnation in Panel C of \cref{fig:length_scaling}. This observation holds consistently across different subspace dimensions $K$. 

Note that, according to an early stopping criterion based on the validation loss, training would typically stop within a range of $10^2$ to $10^4$ epochs. Therefore, the increasing training instability observed in the gradient norm and training loss in Panel C of \cref{fig:length_scaling} is primarily a theoretical concern, as it arises only at around $10^5$ epochs.
\begin{figure}
    \centering
    \includegraphics[width=0.9\linewidth]{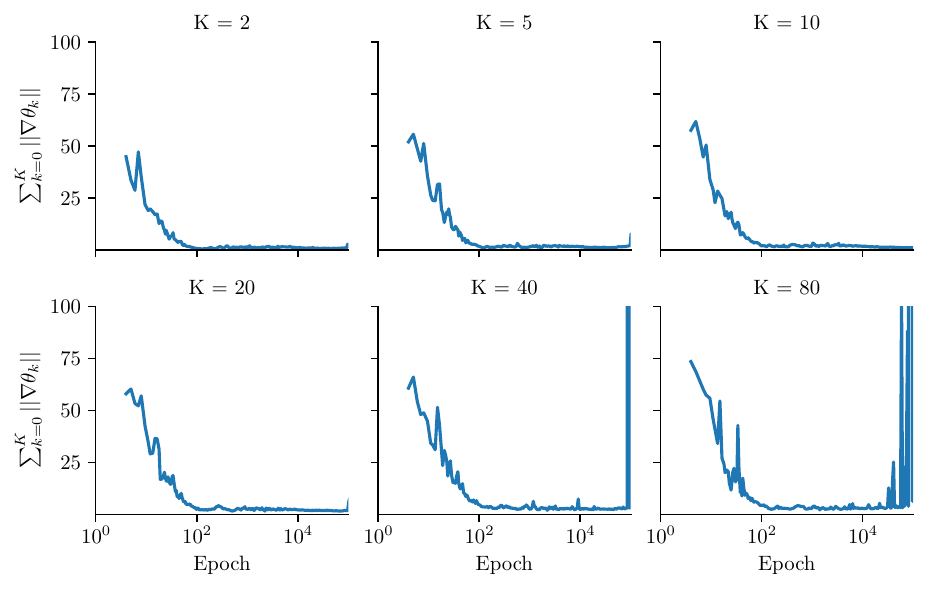}
    \caption{Additional results corresponding to \cref{fig:length_scaling}, visualizing the gradient norm 
    $\sum^K_{k=0} ||\nabla\w_k||_2$ (y-axis) over the course of training epochs (x-axis). 
    The color shading represents the $\pm$ standard deviation, computed from five independent repetitions 
    and logarithmic binning of the epochs.
    }
    \label{fig:length_scaling_gradient_norm}
\end{figure}

Additional results for non-zero weight decay are presented in \cref{fig:scaling_wd_appendix_wd}.
\begin{figure}[h!]
    \centering
    \includegraphics[width=0.9\textwidth]{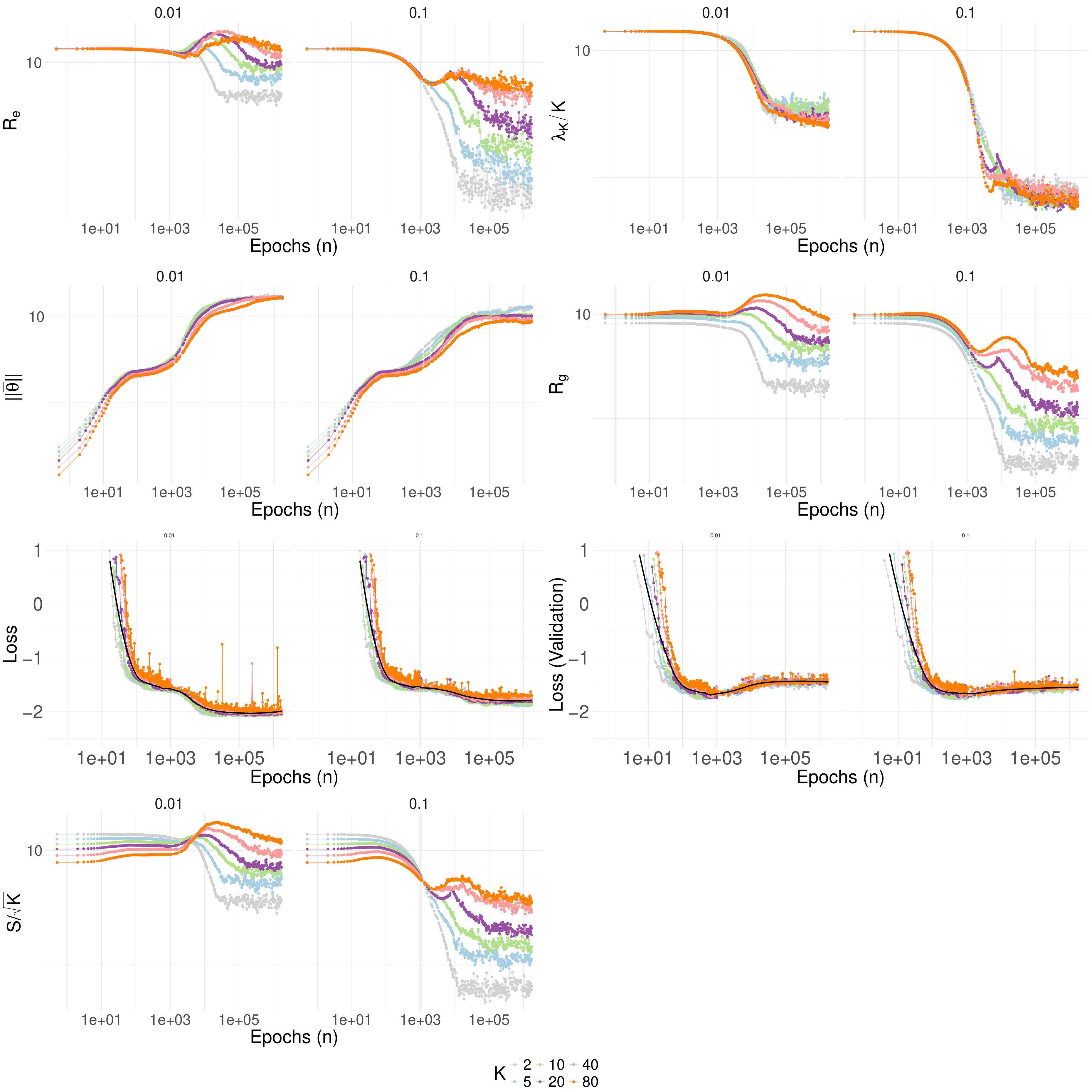} 
    \caption{Additional results for non-zero weight decay $0.01$ and $0.1$}
    \label{fig:scaling_wd_appendix_wd}
\end{figure}

\clearpage

\section{Theoretical Results} \label{app:theory}

\begin{theorem}[Permutation Symmetries]\label{thm:perm-inv}
    Assume each $\w_{\kk}\in \paramset$ satisfies $\w_{\kk,i}<\w_{\kk,i+1}$ for all its entries
    $i=1,...,\D-1$, where $\w_{\kk,i}$ denotes the $i$-th component of $\w_{\kk}$.
    Define $S_\D$ as the permutation group acting on $\mathbb{R}^\D$ by permuting 
    the components. 
    For $\epsilon>0$ define $$\bez_{\paramset}([0,1])_\epsilon:=\{\w\in \mathbb{R}^\D\colon \exists t\in [0,1] \text{ s.t. } ||\bez_{\paramset}(t)-\w||<\epsilon\}$$
    as the ``tunnel'' of radius $\epsilon$ around the curve $\bez_{\paramset}$.
    
    Then, there exists some $\epsilon>0$ such that for all $g\in S_\D$ and 
    all parameters $\w \in \bez_{\paramset}([0,1])_\epsilon$ it holds $g\w\notin \bez_{\paramset}([0,1])_\epsilon$.
\end{theorem}

\begin{proof}[Proof of \Cref{thm:perm-inv}]
    By assumption, it holds $0<\min_{\kk=1,...,K,i=1,...,\D-1}\w_{\kk,i+1}-\w_{\kk,i}$.
    Thus, there exists some $0<\epsilon<\min_{\kk=1,...,K,i=1,...,\D-1}\w_{\kk,i+1}-\w_{\kk,i}/2$.
    It suffices to prove that any $\w\in \bez_{\paramset}([0,1])_\epsilon$ is ordered,
    i.e. $\w_i<\w_{i+1}$ for all $i=1,...,\D-1$.
    Then, any permutation of $\w$ is not ordered and, therefore, not an element of 
    $\bez_{\paramset}([0,1])_\epsilon$.

    Let $\w\in \bez_{\paramset}([0,1])_\epsilon$. 
    By definition, there exists some $t\in [0,1]$ such that $||\bez_{\paramset}(t)-\w||<\epsilon$. 
    In particular,
    \begin{align*}
        \w_{i+1}-\w_{i}>\bez_{\paramset}(t)_{i+1}-\bez_{\paramset}(t)^{i}-2\epsilon.
    \end{align*}
    Note that 
    \begin{equation*}
        \min_{i=1,...,\D-1}\bez_{\paramset}(t)^{i+1}-\bez_{\paramset}(t)^{i}
        \geq 
        % \\\sum_{\kk=0}^\K \binom{\K}{\kk} (1-t)^{\K-\kk} t^{\kk}\min_{\kk=1,...,K,i=1,...,\D-1}\w_{\kk,i+1}-\w_{\kk,i}\geq
        \min_{\kk=1,...,K,i=1,...,\D-1}\w_{\kk,i+1}-\w_{\kk,i}
        >2\epsilon>0
    \end{equation*}
    for all $t$ since $\bez_{\paramset}(t)$ is a finite linear combination of ordered vectors $\w_\kk$ with positive scalars (see \eqref{eq:bezier}) and 
    $\sum_{\kk=0}^\K \binom{\K}{\kk} (1-t)^{\K-\kk} t^{\kk}=1$. 
    Thus,
    \begin{align*}
        \w_{i+1}-\w_{i}>\bez_{\paramset}(t)_{i+1}-\bez_{\paramset}(t)_{i}-2\epsilon >0
    \end{align*}
    which proves the claim. 
\end{proof}

\section{Further Implementation Details} 
\label{app:implem}

All implementation details, including the full set of hyperparameter configurations, model architecture specifications, and training procedures, are also provided in the supplementary code at \href{https://github.com/doldd/Paths-and-Ambient-Spaces}{https://github.com/doldd/Paths-and-Ambient-Spaces}, with a dedicated README file to navigate through the experiments.

\subsection{Sampling Algorithm} \label{app:samplalg}

\cref{alg:subspace_cons_sampling} outlines the implementation of our tunnel-guided sampling routine, using Metropolis-Hastings as an example. The routine is, however, sampler-agnostic, and other sampling methods can also be applied. In our experiments, we primarily use the NUTS \citep{hoffman14a} sampler or the more efficient MCLMC \citep{robnik2023microcanonical} sampler, with the implementation from the BlackJAX library.

\begin{center}
\begin{minipage}{0.6\textwidth}
\centering
\begin{algorithm}[H]
\caption{Sampling}
\label{alg:subspace_cons_sampling}
\begin{algorithmic}
    \STATE \textbf{define} $\btunsp$-space from $\paramset^\ast$ 
    \STATE \textbf{initialize} Local coordinate system for RMF in $\btunsp$-space
    \FOR{\#Chains \& \#Samples}
        \STATE Draw $t$ and $\btun$ proposal
        \STATE Compute $\varw = \proj  g(t,\btun) $
        \STATE Compute likelihood $p(\mathcal{D}|\varw)$ and prior $p(t, \btun)$
        \STATE Adjust prior with $\log | \det({\partial g(t,\btun)} / {\partial t,\btun}) |$
        \STATE Accept/reject proposal according to MH
    \ENDFOR
\STATE \textbf{return} Samples from {$p(t, \btun |\mathcal{D})$}
%\STATE Identify orthogonal directions to construct tunnel...
\end{algorithmic}
\end{algorithm}
\end{minipage}
\end{center}

\section{Experimental Details} \label{app:exp_details}

\subsection{Datasets} \label{app:datasets}

\subsubsection{Simulated Datasets} \label{app:simdata}

Our approach aligns with the approach by \cite{wilson2020bayesian}. Specifically, we randomly initialize neural network weights $\w_\textrm{gen}$, using the same architecture as in our later experiments, and generate the target labels as $y_i = f_{\w_\textrm{gen}}(x_i) + \epsilon$, where $\epsilon \sim N(0, 0.05)$. Both $x$ and the resulting target $y$ are one-dimensional for easy visualization. For the feature $x_\textrm{train}$ of the training dataset, we sample values uniformly distributed between $-2$ and $2$. We exclude values between $-0.6$ and $0.6$ to simulate in-between uncertainty, resulting in 70 training data points. Similarly, we generated 18 validation data points, which are in the same range as the training data, and 33 test data points in an extended range (cf.~\cref{fig:synthetic_data}). This data generation process was repeated with newly initialized $\w_\textrm{gen}$ for each repetition of the synthetic dataset experiment. \cref{fig:synthetic_data} shows all the generated data.

\begin{figure}[ht]
    \centering
    \includegraphics[width=0.9\linewidth]{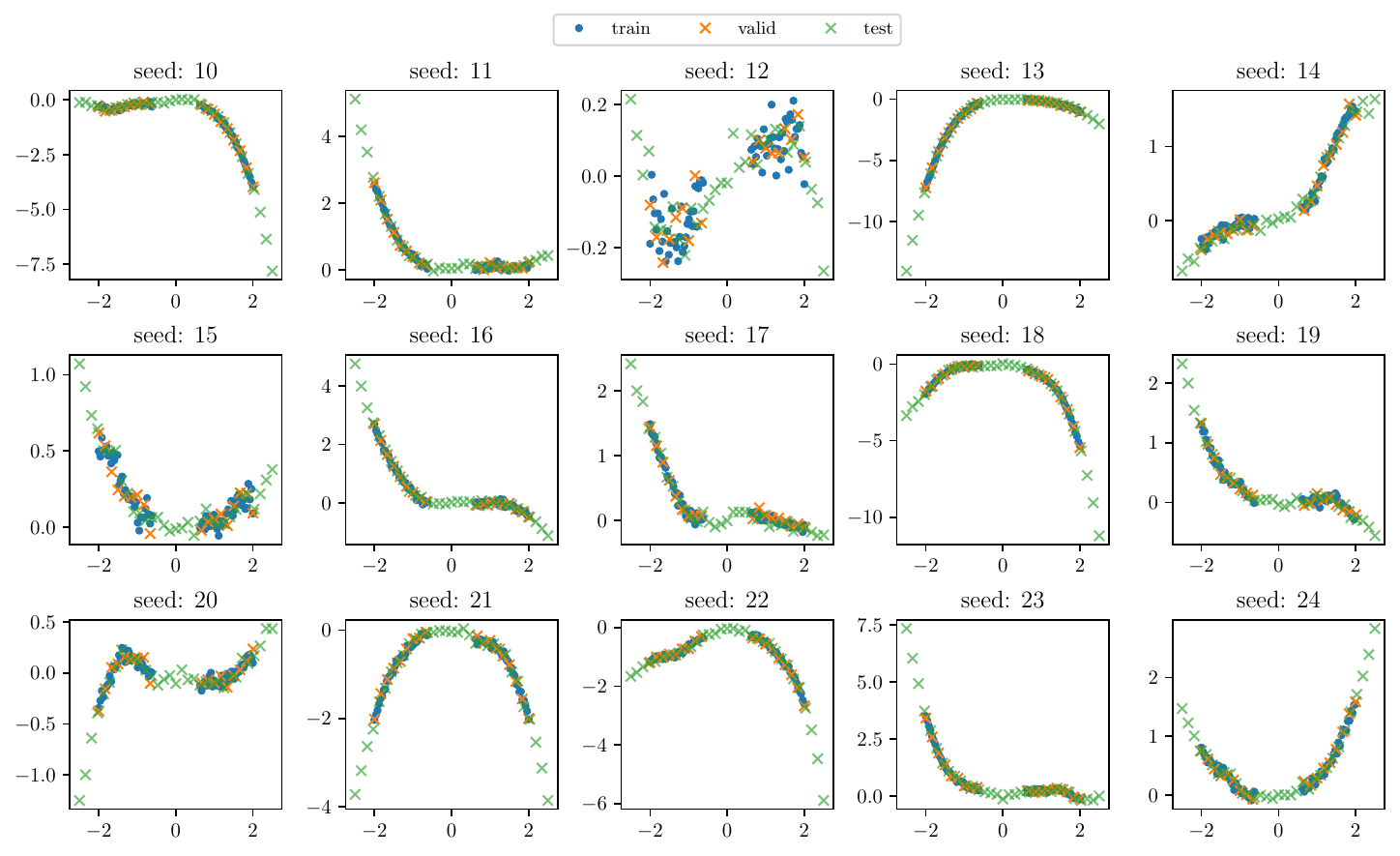}
    \caption{Synthetic generated regression data using a fully connected neural network with three hidden layers, each with 16 neurons and ReLU activation function.}
    \label{fig:synthetic_data}
\end{figure}

\subsubsection{Benchmark Datasets} \label{app:benchdata}

\cref{tab:dataoverview} provides an overview of the benchmark data sets used in our experiments.

\begin{table}[h]
\begin{small}
% \begin{sc}
\begin{center}
\caption{Overview of the used  benchmark datasets} \label{tab:dataoverview}
\vskip 0.1in
\resizebox{0.9\columnwidth}{!}{%
\begin{tabular}{lrrrrl}
Data set & Task & \# Obs. & Feat. & Reference \\ \hline 
airfoil & Regression & 1503 & 5 &  \citet{Dua.2019}  \\
bikesharing & Regression & 17379 & 13 & \citet{misc_bike_sharing_dataset_275}  \\
concrete   & Regression & 1030 & 8 &  \citet{Yeh.1998} \\
energy  & Regression & 768 & 8 &  \citet{Tsanas.2012} \\
yacht  & Regression & 308 & 6 & \citet{Ortigosa.2007,Dua.2019} \\
MNIST & Multi-Class. & 60000 & 28x28 & \citet{lecun-mnisthandwrittendigit-2010} \\
\hline
\end{tabular}
}
\end{center}
% \end{sc}
\end{small}
\end{table}

We standardize the input data for the UCI regression tasks and also compute the LPPD (cf.~\cref{eq:meanlppd}) on the standardized data. For the MNIST dataset, we use 5-fold cross-validation on the train and validation set, where the 60,000 observations were split into one fold of $\num{48000}$ training data and $\num{12000}$ validation data, following their original order. The LPPD and RMSE for MCMC differ from the values reported in \citep{sommer24connecting} because they used a heteroscedastic approach, whereas we modeled the outcome distribution in a homoscedastic fashion. This choice was made because our focus is not on performing distributional regression, and heteroscedastic optimization comes with its own set of challenges.

%For all datasets, the validation data is used to select the curve parameters $\paramset^\ast$ that correspond to the best-performing path.

\subsection{Methods}

\paragraph{Methods for \cref{sec:exp_tunnelVsVolume} to \cref{sec:regression_benchmarks}} 

In experiments of \cref{sec:exp_tunnelVsVolume,sec:regression_benchmarks}, our {model is defined} as a fully connected network with three hidden layers, each with 16 neurons. As an activation function, we use ReLU for the regression benchmarks in \cref{sec:regression_benchmarks}. In \cref{sec:exp_tunnelVsVolume} and \cref{sec:exp_tunnelSmoothness}, we use both ReLU and ELU activations. The only difference in the model architecture for the simulated data and the UCI benchmark data is the inclusion of a feature expansion for the simulated data, where we extended $x$ to the feature vector $[x, x^2, x^3]$.

For the Laplace Approximation (\textbf{LA}) method, we first train MAP solutions using the Adam optimizer with regularization for $\num{10000}$ epochs with a learning rate of $0.005$ and then use last-layer LA with a generalized Gauss-Newton Hessian approximation and a closed-form predictive approximation over five different initializations and splits \citep{daxberger2021laplace}. \\

For the deep ensembles (\textbf{DE}), we train 10 MAP solutions with different initializations using the training recipe described above for LA and average the predictions to obtain the performance results.\\

For the \textbf{ModeCon} method, we implement the approach described in \citet{garipov2018loss}, where we train two MAP-solutions $\w_{0}, \w_{2}$ with different initializations as start- and endpoint of a B\'{e}zier curve and add a third randomly initialized model, $\w_{1}$, to obtain a quadratic B\'{e}zier curve with three control points, given by $\w_{0}, \w_{1}, \w_{2}$. In the subsequent curve training step, we keep the endpoints (MAP solutions) fixed and only apply loss-gradient-based updates to $\w_{1}$. We use early stopping by evaluating the path loss $L(\w_{0}, \w_{1}, \w_{2})$ at 50 uniformly spaced-out positions of $t$. After training the curve, we employ volume lifting as described in Section \ref{sec:tunnel_emb_lifting} and use the NUTS sampler to obtain posterior samples from the hyperplane. We use an isotropic Gaussian prior for the parameters of the hyperplane, consistent with \citet{izmailov2020subspace}. Last, we define the orthogonal projection matrix $\bm{\Pi}$ using a principal component analysis of the centered control points $\w_{0}, \w_{1}, \w_{2}$, keeping the first two principal components to project the samples from the hyperplane to a neural network weight $\w\in\W$.\\

For our \textbf{Tunnel-2} method, we first initialize $\paramset = \{\w_0, \w_1, \w_2\}$ randomly and optimize these parameters jointly as described in \cref{alg:subspace_cons}. We use the Adam optimizer and manually adjust the learning rate to stabilize the training loss. For the simulated data, we choose a learning rate of 0.001. For the UCI benchmark, the learning rate is set to 0.005. We run the entire training process for $\num{100000}$ epochs and the optimal parameters $\paramset^\ast$ are selected based on the validation performance of the curve. To compute the validation performance, we approximate the overall curve performance by averaging the log-likelihood evaluated at 1,000 evenly spaced grid points between zero and one for the curve parameter $t$.

For all consecutive sampling evaluations (tunnel vs.\ volume lifting, different temperatures), we use the same optimized path defined by $\paramset^\ast$, ensuring that no differences arise from variations in the optimized path. We use the MCLMC sampler from \cite{robnik2023microcanonical,sommer2025microcanonical} with 10 chains, 1000 warm-up samples, and 1000 collected samples. To reduce autocorrelation between samples and maintain comparability with NUTS, we increase the total number of samples by a factor of 100 but retain only every 100th sample. This simulates the leapfrog integration used in NUTS while maintaining a fixed computational budget.\\

In our \textbf{Tunnel-K} method, the setup is the same as in the Tunnel-2 configuration, with the difference that $\paramset$ consists of $K+1$ parameters instead of three, and a ($K-1$)-dimensional $\btun$ during sampling.
 \\

\paragraph{Methods for \cref{sec:exp_mnist}} 

For the \textbf{LA} method, we train a MAP solution using the Adam optimizer with regularization for $\num{5}$ epochs with a batch size of $\num{64}$ and a learning rate of $0.001$. We use a last-layer LA with a generalized Gauss-Newton Hessian approximation and the Monte Carlo predictive approximation over five different initializations and splits \citep{daxberger2021laplace}. \\

For the \textbf{Tunnel-$K$} method, we optimize the path using an Adam optimizer with a batch size of 480, over 50 epochs, and a learning rate of 0.005. We reduce the number of collected samples to 500, with 500 warm-up samples and 10 chains. Additionally, we add 10 extra steps per sample to simulate the behavior of the NUTS sampler within the MCLMC sampler. All other configurations remain the same as in the other experiments. %, except for reduced vectorized operations to decrease memory consumption.

\subsection{Performance Evaluation}

To evaluate the performance of the subspace inference approaches on a set of previously unseen test data, we use the log pointwise predictive density (LPPD), which takes into account the posterior distribution \citep{gelman2014understanding}:
\begin{equation}
\text{LPPD} = \sum_{i=1}^{n_{obs}} \log \int p(y_i | \bm{\theta} ) p(\bm{\theta} | y_i ) d\bm{\theta}
\end{equation}
In practical sampling-based inference, we evaluate the expectation using (vector-valued) draws from the posterior, $\bm{\theta}^{(u)}, u = 1, \dots, U$, and normalize the $\text{LPPD}$ by the number of samples $n_{obs}$ to obtain a metric that is independent of the dataset size: 
\begin{equation} \label{eq:meanlppd}
     \widehat{\text{LPPD}} = \frac{1}{n_{obs}} \sum_{i=1}^{n_{obs}} \log \left(\frac{1}{U} \sum_{u=1}^{U} p(y_i | \bm{\theta}^u) \right)
\end{equation}

\subsection{MCMC Sampling Evaluation}
\label{app:sampling_eval}

For the evaluation of the drawn posterior samples, we use the effective sample size \citep[ESS;][]{vehtari2021rank} and \^{R} \citep{gelman2014understanding}, where ESS is given by
\begin{equation}
    \text{ESS} = \frac{U}{1 + 2 \sum_{t=1}^{\infty} \rho_t}
\end{equation}
and $\rho_t$ denotes the autocorrelation in a single MCMC chain at different lags $t$ \citep{vehtari2021rank, stanrefmanual}. The ESS of an MCMC chain quantifies the number of independent samples that an autocorrelated sequence effectively represents. %While a chain may contain many samples, autocorrelation reduces the information it provides, inflating uncertainty in parameter estimates. ESS adjusts the nominal sample size by accounting for the autocorrelation within the chain, providing a more accurate measure of the information content and reliability of the estimates derived from the MCMC process \citep{stanrefmanual}. As discussed in \citet{papamarkou2022challenges} and \citet{sommer24connecting}, ESS values are typically low for BNNs without apparent structure, which leads \citet{sommer24connecting} to subsume that ESS might not be a fitting metric for BNN sampling evaluation, as we would expect a lot of autocorrelation in this overparameterized setting. \\

The potential scale reduction metric $\hat{R}$ is a convergence measure that compares within-chain variance $W$ and between-chain variance $B$ to assess chain convergence. When all chains have reached equilibrium, these variances will be equal, and $\hat{R}$ will be equal to $1$. However, if the chains have not converged to a common distribution, the statistic will exceed one \citep{stanrefmanual}. 
Following the definition put forth by the \citet{stanrefmanual}, on $M$ different chains, each chain containing $U$ samples $\theta^{(u)}_m$, the between-chain variance is defined as
\begin{equation}
    B =
\frac{U}{M-1}
\,
\sum_{m=1}^M (\bar{\theta}^{(\bullet)}_{m}
                - \bar{\theta}^{(\bullet)}_{\bullet})^2
\end{equation}
where 
\begin{equation}
\bar{\theta}_m^{(\bullet)}
= \frac{1}{U} \sum_{u = 1}^U \theta_m^{(u)} 
\quad \text{and} \quad
\bar{\theta}^{(\bullet)}_{\bullet}
= \frac{1}{M} \, \sum_{m=1}^M \bar{\theta}_m^{(\bullet)}.
\end{equation}
The within-variance is given by
\begin{equation}
    W = \frac{1}{M} \, \sum_{m=1}^M v_m^2 \quad \text{with} \quad
    v_m^2
=
\frac{1}{U-1}
\, \sum_{u=1}^N (\theta^{(u)}_m - \bar{\theta}^{(\bullet)}_m)^2.
\end{equation}
The variance estimator $\widehat{\mbox{var}}^{+}\!(\theta|y)$ is a mixture of the within-chain and cross-chain sample variances,
\begin{equation}
    \widehat{\mbox{var}}^{+}\!(\theta|y)
= \frac{U-1}{U}\, W \, + \, \frac{1}{U} \, B,  \quad \text{with the final metric defined as} \quad \hat{R}
\, = \,
\sqrt{\frac{\widehat{\mbox{var}}^{+}\!(\theta|y)}{W}}.
\end{equation}
%
%
%The metric was initially designed for identifiable problems where, additionally, the unknown parameter posterior distribution is unimodal. The use of this metric is sensible in the context of subspace inference if we assume the posterior samples to stem from the optimized path valley and thus proves helpful in assessing path volume posterior sampling. Since the path volume topology can be quite intricate, we argue that chain convergence metrics are indispensable for practical subspace inference, as they allow conclusions about how the sampling process can make use of the diverse functions condensed in the curve subspace, e.g., by exploration behavior of different sampler chains.

\subsection{Permutation Symmetries}
\label{app:permutation_symmetry}

As symmetric solutions would not differ in their functional outcomes, we inspect the functional outcome captured by the path along the curve in \cref{fig:functional_diversity_on_synthetic_data}. The figure shows that our path does not collapse to one functional realization and the functional outcome differs along the curve. We also observe that a smooth transition occurs in the functional outcome as we travel along the curve. 
%This indicates that even small changes in the parameter vector $\varw$, resulting from small changes in $t$, do not correspond to symmetric solutions.
%Therefore, we analyzed the functional diversity captured by the path along the curve, as shown in \cref{fig:functional_diversity_on_synthetic_data}. Since symmetric solutions would produce the same functional outcomes, this figure demonstrates that our path does not consist solely of symmetric solutions, as evidenced by the observed functional diversity. 

Additionally, \cref{fig:functional_diversity_on_synthetic_data} shows that increasing the subspace dimension $K$ can improve the functional diversity captured by the path. Comparing the different path objectives, as discussed in \cref{sec:pathoptim}, does not reveal significant differences and is included primarily for completeness.

\begin{figure}
    \centering
    \includegraphics[width=0.8\linewidth]{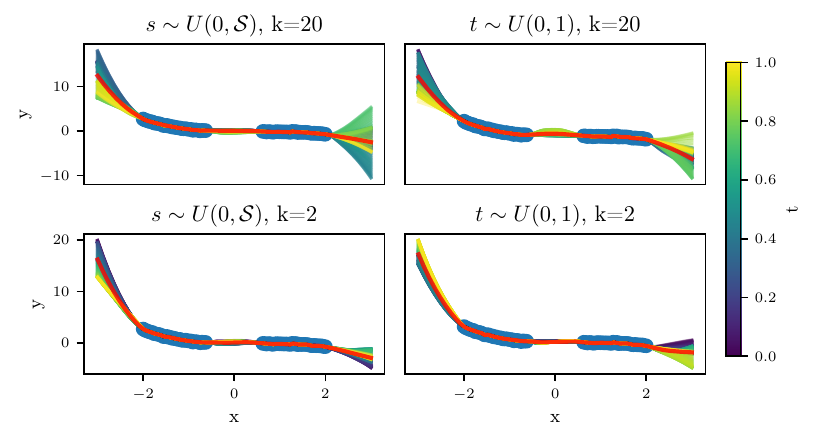}
    \caption{Investigating functional diversity on an exemplary synthetic dataset generated with ELU activation and seed $=16$. The columns compare two different path objectives: $t \sim U(0,1)$ and $s \sim U(0, \pathlen)$, while the rows compare subspace dimensions $K=2$ and $K=20$. The blue dots represent the training data, and the red line shows the averaged prediction along the entire curve. Each color corresponds to a single prediction at a specific $t$ or $s$ value. For simplicity, only the color legend for $t$ values is shown; the same color scheme applies to $s$.}
    \label{fig:functional_diversity_on_synthetic_data}
\end{figure}

To validate our \cref{thm:perm-inv}, we modify the MLP network to implicitly incorporate bias parameters that are sorted in each layer to obtain permutation symmetry-free solutions according to \cite{pourzanjani_2017_ImprovingIdentifiabilityb}. Sorting is achieved by optimizing the unrestricted bias parameter $\bm{B^\prime}$  for every control point $\w_k$ and converting it into the original bias parameter $\bm{B}$ using 
\begin{equation}
B_i = \sum_{j=0}^{i} \tilde{B}_j, \quad \text{where} \quad \tilde{B}_j = \begin{cases} B^{\prime}_0, & j = 0 \\ \ln(1 + e^{B^{\prime}_j}), & j > 0. \end{cases}
\end{equation}
The results of this experiment are visualized in \cref{fig:bias_ordering_of_bias_sorted_architecture}. We observe that the bias parameters are arranged in ascending order along the entire path defined by the B\'{e}zier curve. This indicates that adjusting only the control points results in a permutation-free path. This finding aligns with the theoretical framework established in \cref{thm:perm-inv}.

\begin{figure} 
    \begin{minipage}{0.9\textwidth} 
        \centering 
        \includegraphics[width=0.4\linewidth]{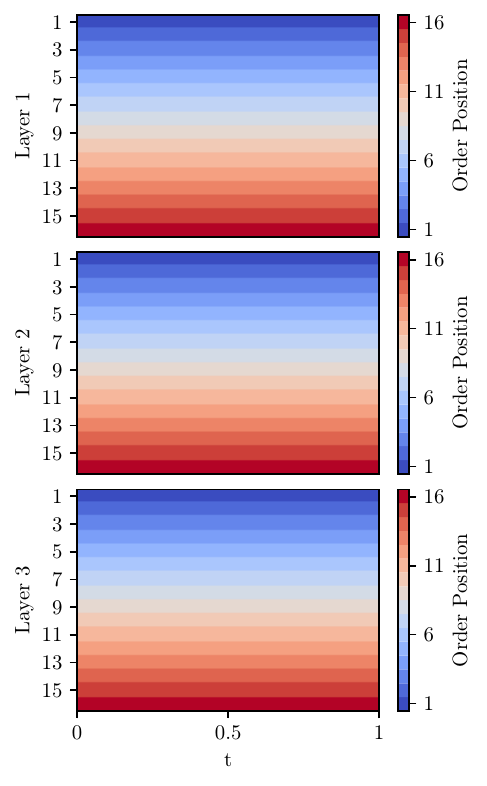} 
        \includegraphics[width=0.4\linewidth]{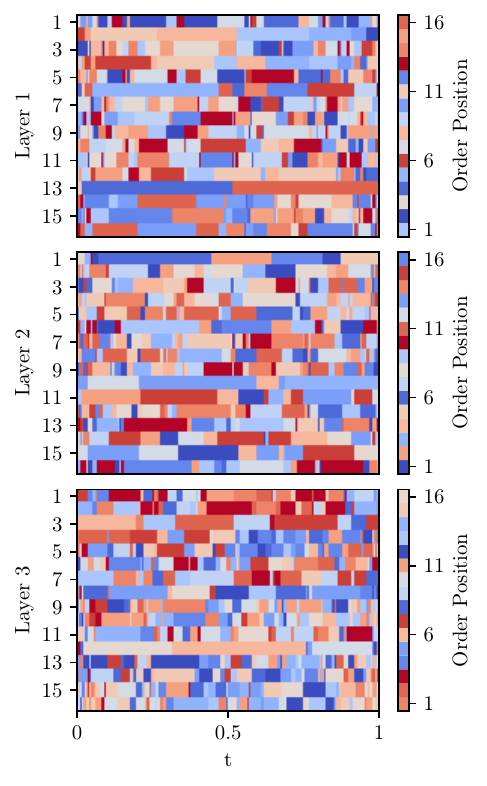}
    \end{minipage} 
    \caption{Left column: Bias ordering of a three-hidden-layer MLP with 16 neurons per layer. The Y-axis represents neurons in each layer (rows), color-coded according to their sorting index, while the X-axis represents the traversal along the path.
    Right column: Bias ordering in each layer when traversing the path without adjustment.} \label{fig:bias_ordering_of_bias_sorted_architecture} 
\end{figure}

To further investigate how the adjustment of permutation symmetry influences performance, we compared both approaches --- one with the adjustment of permutation symmetry by bias sorting and one without --- in the regression data set used in \cref{sec:exp_tunnelVsVolume}. In \cref{fig:valid_performance_bias_sorting}, we observe that validation performance slightly decreases when enforcing a sorted bias parameter in each layer. This suggests that even without adjustment, the path does not contain harmful permutation symmetries. If such symmetries were present, eliminating all permutation-symmetric solutions would be expected to improve performance.

\begin{figure} 
    \centering 
    \includegraphics[width=\linewidth]{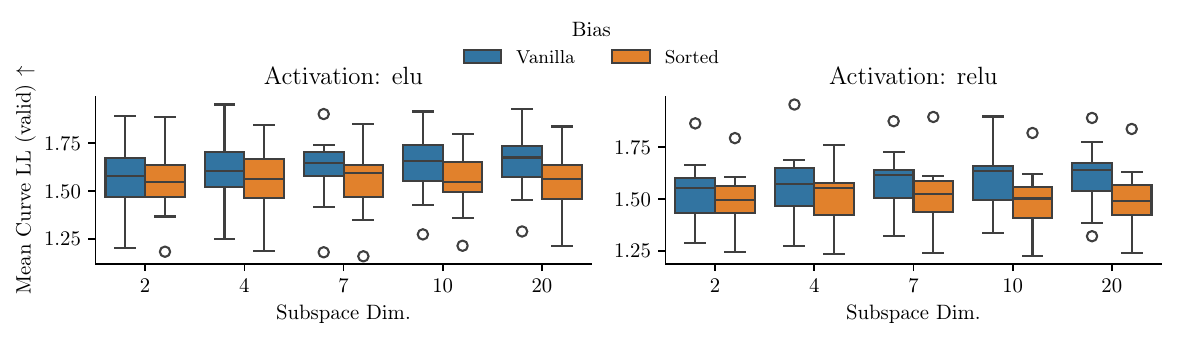} 
    \caption{Averaged log-likelihood over the path on the validation data for increasing subspace dimension (x-axis) and different activation functions (columns). The different hues compare cases where the bias parameter was sorted (orange) versus left unchanged (blue).} \label{fig:valid_performance_bias_sorting} 
\end{figure}

\subsection{Effects of Temperature}

%Besides introducing a tunnel prior,  it can be beneficial to incorporate a temperature in the sampling process to better guide the sampler through the tunnel. 
In addition to introducing a tunnel prior, incorporating a temperature parameter in the sampling process can help in better guiding the sampler through the tunnel.
%This adjustment is crucial when employing the $\phi$-space parameterization.
%More specifically, by down-weighting the influence of the likelihood during sampling, we can potentially achieve a better exploration of the tunnel since the posterior distribution tends to become overly concentrated around the maximum likelihood estimate (MLE). 
More specifically, by down-weighting the influence of the likelihood during sampling, we can facilitate better exploration of the tunnel, as the posterior distribution often becomes overly concentrated around the maximum likelihood estimate (MLE). 
This excessive concentration, in turn, can lead to overly confident uncertainty estimates by failing to explore the whole tunnel. 
As in other practical applications of temperature, where it is often found to yield better predictive accuracy if the likelihood is tempered with $\frac{1}{T}$ \citep{izmailov2020subspace}, we implement temperature scaling of the likelihood as:
$p_{\text{temp}}(\bm{\theta}|\mathcal{D}) \propto p(\mathcal{D}|\bm{\theta})^{1/T}p(\bm{\theta})$ or, expressed in log-space,
\begin{equation}
    \log p_{\text{temp}}(\bm{\theta}|\mathcal{D}) = \frac{1}{T} \log p(\mathcal{D}|\bm{\theta}) + \log p(\bm{\theta}) + \text{const}.
\end{equation}
For $T=1$, the original posterior is recovered. As $T \rightarrow \infty$, $p_{\text{temp}}(\bm{\theta}|\mathcal{D})$ converges to the prior. The effect of $T$ on $\log p_{\text{temp}}(\bm{\theta}|\mathcal{D})$ is visualized in \cref{fig:path_temp}.

\begin{figure}[ht]
    \centering
    \includegraphics[width=.5\linewidth]{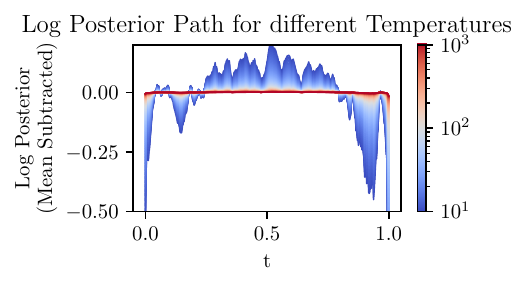}
    \caption{Unnormalized log posterior along the path for different temperature parameters using the B\'{e}zier curve trained on the yacht dataset. The averaged unnormalized log posterior along the curve was subtracted for each specific temperature value.}
    \label{fig:path_temp}
\end{figure}

\section{Additional Numerical Results} \label{app:furthres}

% \subsection{Path Length} \label{app:pathlen}

% [Description Experiment]

% \begin{figure}[ht]
%     \centering
%     \label{fig:length_vs_epoch}
%     \includegraphics[width=0.4\linewidth]{figures_dd/bezier_length_vs_epoch.pdf}
%     \caption{Visualizing the curve length evolution during six repeated evaluations with a subspace dimension $k=20$. The bold line depicts the mean, while the shaded areas indicate the 95\% confidence interval, estimated through bootstrapping.}
%     \label{fig:enter-label}
% \end{figure}

%\subsection{G Radius} \label{app:curve}

\subsection{Synthetic dataset}

\cref{fig:additional_synthethic_lppd_res} provides additional results for the experiments discussed in \cref{sec:exp_tunnelVsVolume}. While in \cref{sec:exp_tunnelVsVolume}, the temperature parameter is selected based on the best LPPD performance on the validation data, \cref{fig:additional_synthethic_lppd_res} shows the LPPD performance across all temperature values. Notably, the optimal temperature depends on the subspace dimension $K$, with larger $K$ values requiring lower temperatures. This observation aligns with the motivation in \cite{izmailov2020subspace}, which suggests adding temperature to account for the dimensional difference between the lower-dimensional subspace, where the prior is defined, and the high-dimensional weight space. As $K$ increases, these differences decrease, thus requiring smaller temperature parameters. 

The comparison between the tunnel and volume-lifting approaches shows that the tunnel prior is more appropriate, as it performs better in the high-temperature range. Additionally, comparing networks with ReLU vs.\ ELU activations reveals only a minor difference in performance.

\begin{figure}[ht]
    \centering
    \includegraphics[width=1.\linewidth]{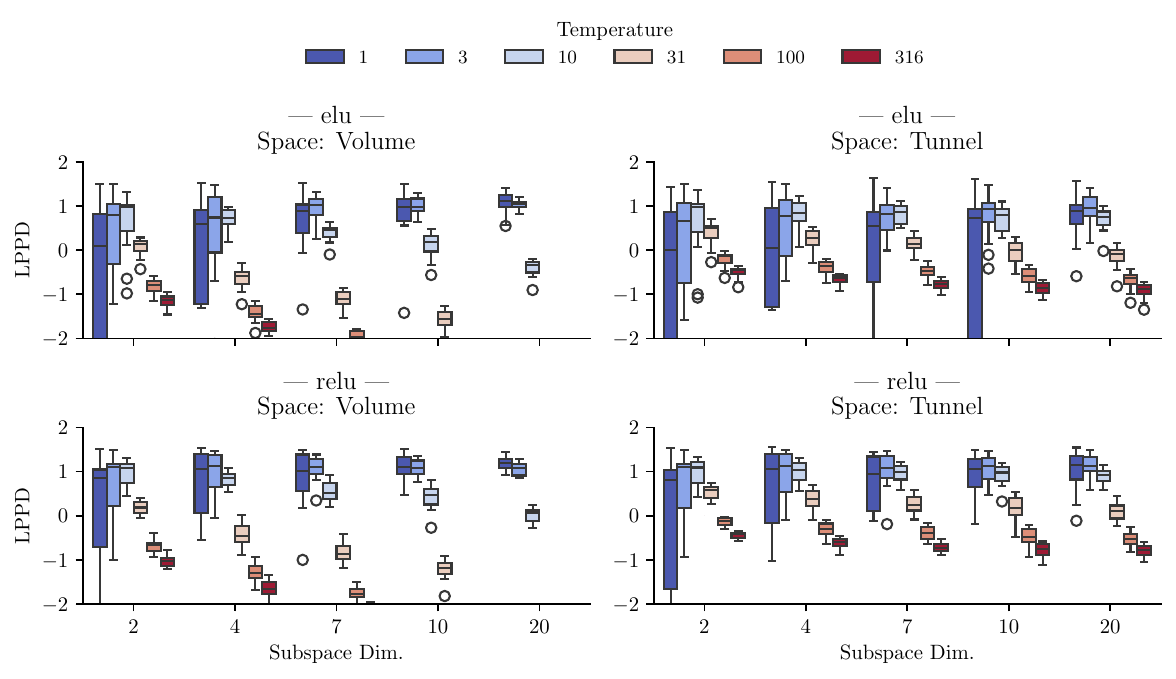}
    \caption{LPPD performance on the synthetic dataset using different temperature parameters. Each row compares a fully connected neural network with ELU activation versus ReLU activation and each column compares the volume versus tunnel-lifting approach. }
    \label{fig:additional_synthethic_lppd_res}
\end{figure}

\subsection{Benchmark dataset}

\cref{fig:exp_uci_all_metrics} shows different evaluated metrics for the benchmark dataset, depending on the temperature (x-axis), and compares the $K=2$ with the  $K=20$ tunnel lifting (colors). We see the same behavior as in \cref{fig:additional_synthethic_lppd_res}, i.e., $K=20$ requires lower temperatures. 

\begin{figure}[ht]
    \centering
    \includegraphics[width=.9\linewidth]{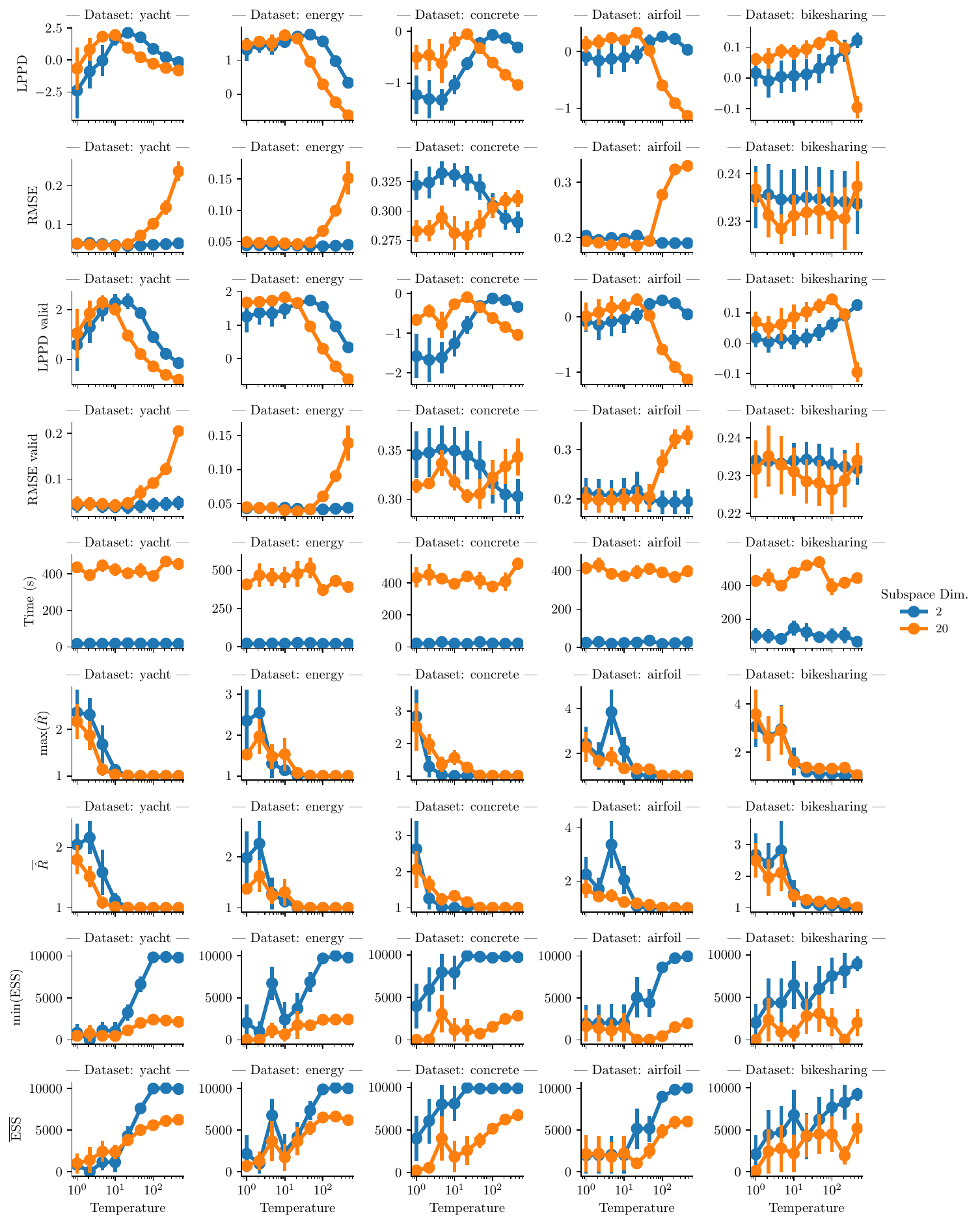}
    \caption{Different computed metrics (y-axis) versus temperature parameter (x-axis) for different datasets (columns). Colors represent the subspace dimension: $K=2$ (blue) and $K=20$ (orange). The error bars indicate the standard error across five repetitions. We compute the average and worst metrics for $\textrm{ESS}$ and $\hat{R}$ across all model parameters. The metric ``Time (s)'' represents the wall clock time, measured in seconds, to sample 10 chains with 1000 samples per chain.}
    \label{fig:exp_uci_all_metrics}
\end{figure}

\clearpage

\subsection{MNIST}
\label{app:mnist_temperature}

\cref{fig:exp_mnist_all_temps} presents the same evaluation as in \cref{fig:mnist}, but instead of using the best temperature selected based on the validation LPPD, the behavior across specific temperature parameters is visualized. We observe a trade-off between optimizing the temperature according to the LPPD metric and the accuracy, with accuracy favoring lower temperature values.

\begin{figure}[ht]
    \centering
    \includegraphics[width=0.6\linewidth]{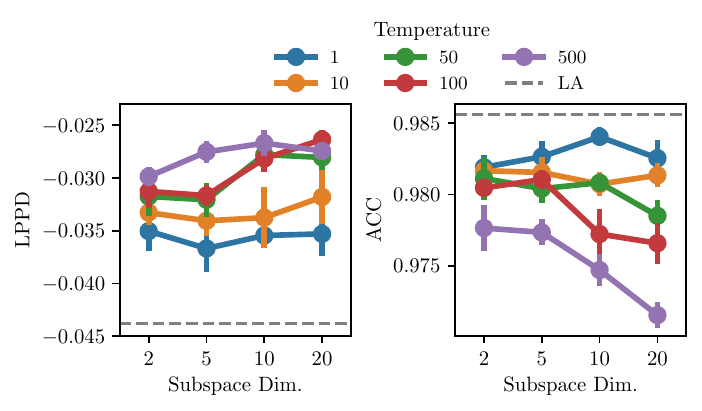}
    \caption{Performance on MNIST (with the standard error~over five different folds) with different temperature parameters (color) and subspace dimensions $\K$ (x-axis). The dashed line shows the results using Laplace Approximation (LA). The underlying path is optimized with $t \sim U(0,1)$ sampling.}
    \label{fig:exp_mnist_all_temps}
\end{figure}

\section{Computational Environment}\label{app:compenv}

The complexity of our dataset was intentionally kept manageable, allowing us to conduct experiments on consumer hardware, such as an NVIDIA 3080Ti GPU with 11GB of VRAM. While the 11GB of VRAM was not essential, it did improve runtime by enabling parallel execution of multiple chains during the sampling process. Only for the MNIST dataset, where we performed full-batch evaluations during sampling, we resorted to a Titan RTX GPU with 24GB of VRAM to accommodate the increased memory requirements.

From an implementation perspective, our experiments were conducted using the JAX library \citep{jax2018github}, in combination with Flax to specify the model architecture. The same model description was used to define our sampling model within the NumPyro framework, and sampling was performed using BlackJAX, which specified the sampler. For reproducibility, all runs were tracked using a local instance of ``Weights \& Biases'', which was also used to collect the data needed to generate the figures.

\end{document}